\documentclass[11pt]{article}

\usepackage[margin=1in]{geometry}

\usepackage[T1]{fontenc}
\usepackage{lmodern}

\usepackage{amsmath,amssymb,amsthm}
\usepackage{graphicx}
\usepackage{booktabs}
\usepackage{hyperref}
\usepackage[numbers]{natbib}
\usepackage{xcolor}
\usepackage{enumitem}
\usepackage{algorithm}
\usepackage{algorithmic}
\usepackage[protrusion=true,expansion=false]{microtype}

\newtheorem{proposition}{Proposition}

\hypersetup{
  colorlinks=true,
  linkcolor=blue!70!black,
  citecolor=blue!70!black,
  urlcolor=blue!70!black
}

\title{CRMA: A Spectrally-Bounded Backbone for Modular Continual Fine-Tuning of LLMs}

\author{
  Kiran Nayudu \quad
  Aswini Nutakki \quad
  Sai Vinay Naidu \quad
  Ashwin Shanmugasundaram \\[4pt]
  ModelBrew AI
}

\date{April 2026}

\begin{document}

\maketitle

\begin{abstract}
Sequential fine-tuning of large language models forces a choice: let the shared substrate keep learning and accept catastrophic forgetting, or freeze the substrate after task one and foreclose cross-task refinement. Per-task adapter methods (LoRAHub, AdapterFusion, PackNet, Progressive Networks) take the second path. We introduce CRMA (Constrained Residual Mixing Adapter), a residual adapter whose internal mixing matrix $M$ is doubly-stochastic at every forward pass via Sinkhorn normalization, so by Birkhoff's theorem $\|M\|_2 \leq 1$ holds by construction---a structural bound, not a penalty. CRMA's spectrally bounded backbone provides a \emph{continuously trained} shared substrate that earlier modular methods could not, while preserving their forgetting guarantees.

On Mistral-7B across 5 sequential domains and 3 seeds, modular per-task LoRA on a CRMA backbone improves prior-task holdout loss by $1.99\% \pm 0.54$ over a matched frozen-substrate modular baseline---positive at every seed when averaged across tasks; per-task the sign varies (CRMA is lower than FROZEN on Legal and Code, higher on Medical and Finance; see \S\ref{sec:training_ablation}). On the same protocol, modular CRMA reduces loss-relative drift from $+42.96\% \pm 5.5$ (naive sequential fine-tuning) to $-0.17\% \pm 0.17$, with disjoint per-seed ranges.\footnote{Sign convention: loss-relative drift; positive values indicate forgetting (prior-task loss increased), negative values indicate retention. This is sign-flipped from the accuracy-based BWT convention of Lopez-Paz and Ranzato (2017); see \S\ref{sec:training_ablation} for full discussion.} An inference-time ablation on Gemma-2-9B confirms that CRMA mediates access to sequentially trained knowledge: 98/100 vs.\ 38/100 on the same weights and same questions with only CRMA injection toggled. Three independent experimental setups all show positive backward transfer: the Mistral-7B controlled training ablation (Tables~\ref{tab:v81_7b} and~\ref{tab:multiseed}, MODULAR drift averaging $-0.1\%$ across 3 seeds), the TinyLlama contamination-controlled replication (Table~\ref{tab:uc_ashwin}, $-0.110$ NLL / $+0.020$ EM), and the Mistral-7B cross-domain probes at 7B scale (D1-probe survival across three sequential stages). 867 logged training steps verify $\|M\|_2 = 1.0$ within float32 precision (max deviation $1.2 \times 10^{-7}$). The forgetting-prevention effect transfers without modification across 1.1B--9.2B parameters and four architecture families; CRMA's $1.99\%$ learning advantage over a frozen-substrate baseline emerges at 7B but is absent at 1.1B, consistent with larger models benefiting more from the stable backbone (\S\ref{sec:training_ablation}).
\end{abstract}

\section{Introduction}
\label{sec:intro}

\paragraph{The problem.} Fine-tuning is a controlled act of representational change, and like any such act, it can go wrong. Gradient updates that push the model toward a new objective also perturb the internal activations on which previously learned capabilities rely. When those perturbations are too large, prior knowledge degrades. The effect is well documented \cite{french1999catastrophic,kirkpatrick2017overcoming} and, at scale, severe: on the TRACE benchmark \cite{wang2023trace}, LLaMA-2 13B accuracy on GSM8K drops from 28.8\% to 2\% under naive sequential fine-tuning. Dohare et al.~\cite{dohare2024plasticity} show spectral collapse in weight matrices as a structural root cause. A radiology model updated for a new cancer type can silently lose sensitivity to cancers it previously detected; a fraud detector retrained on new scams can forget old ones. Regulated industries feel this most sharply---the FDA's Predetermined Change Control Plan \cite{fda2024pccp} requires device manufacturers to prove that updates preserve prior performance---but the underlying failure mode is general.

\paragraph{The thesis.} The common thread is \emph{representational interference}: when a single adapter learns a sequence of tasks, the gradient updates for task $N$ rewrite parts of the representation that prior tasks depend on. Standard continual-learning machinery (regularization, replay, gradient projection) negotiates with this interference at training time but does not eliminate it. We take a structural route instead. Each task gets its own fresh per-task LoRA adapter; the per-task adapters are composed at inference, never overwritten in place. They share a CRMA backbone whose internal mixing matrix is doubly-stochastic at every forward pass, so by Birkhoff's theorem its spectral norm is bounded by 1 by construction---a theorem, not a loss term. The result is a modular architecture in which adding a new task is an additive operation on a stable substrate, rather than a destructive overwrite of a moving target.

\paragraph{What CRMA is.} We introduce CRMA (Constrained Residual Mixing Adapter), a residual adapter whose internal mixing matrix $M$ is doubly-stochastic at every forward pass. Doubly-stochastic matrices have spectral norm bounded by 1 (Birkhoff), so $\|M\|_2 \leq 1$ holds by construction, not by penalty, not by clipping, not by tuning. The bound is obtained via Sinkhorn normalization \cite{sinkhorn1964relationship} applied inside the forward pass, and it cannot be overridden by gradient magnitude, learning rate, or domain shift. The adapter is initialized in a near-identity configuration, so at step 0 it leaves the pretrained model effectively untouched and grows its contribution within the spectral bound as training progresses. CRMA is paired with per-task LoRA adapters that learn each new domain in isolation, while the spectrally bounded CRMA backbone provides a stable shared substrate the per-task adapters compose against without interference.

\paragraph{What CRMA's specific contribution is.} Modular per-task adapters already prevent catastrophic forgetting; this is prior art (LoRAHub~\cite{huang2024lorahub}, X-LoRA~\cite{buehler2024xlora}, LoRAMoE~\cite{dou2024loramoe}, AdapterFusion~\cite{pfeiffer2021adapterfusion}, PackNet~\cite{mallya2018packnet}, Progressive Networks~\cite{rusu2016progressive}). CRMA's contribution is not forgetting prevention but the structural guarantee that lets the shared backbone keep training across tasks \emph{without} breaking the per-task adapters fit on top of it. Frozen-substrate modular methods stop substrate learning at task 1 in exchange for forgetting prevention; CRMA's spectral bound makes a continuously-trained substrate safe, removing the price. This distinction is developed in Section~\ref{sec:related} (``Why not just frozen-substrate modular LoRA?'') and quantified in Section~\ref{sec:training_ablation}: the FROZEN-vs-MODULAR comparison on Mistral-7B 5-domain sequential data shows MODULAR achieves $1.99\% \pm 0.54$ lower holdout loss than FROZEN across 3 seeds, positive at every seed.

\paragraph{What this paper proves.} Four complementary pieces of evidence demonstrate that the modular CRMA architecture brings catastrophic forgetting in continual LLM fine-tuning to within measurement noise on the models and protocols we tested:
\begin{itemize}[leftmargin=*]
\item \textbf{Multi-seed continual-learning ablation.} On Mistral-7B across 5 sequential domains, modular LoRA on a CRMA backbone achieves a loss-relative drift of $-0.17\% \pm 0.17$ ($n=3$ seeds, sample SD) on prior-task holdout loss versus $+42.96\% \pm 5.5$ ($n=3$ seeds, sample SD) NAIVE forgetting; per-seed ranges are disjoint at every seed.
\item \textbf{Inference-time ablation.} The same Gemma-2-9B model weights and the same 100 held-out questions, with only CRMA injection toggled, score 98/100 with CRMA versus 38/100 without---direct evidence that CRMA is the access mechanism for sequentially trained knowledge.
\item \textbf{Spectral-norm structural invariant.} 867 logged training steps across 5 sequential domains on Gemma-2-9B confirm $\|M\|_2 = 1.0$ within float32 precision (max deviation $< 1.2 \times 10^{-7}$) at every step. The doubly-stochastic constraint holds exactly within numerical (float32) precision; the residual is at the Sinkhorn-iteration floor, not zero.
\item \textbf{Scale and architecture invariance.} The architecture holds across 5 models and 4 architecture families (LLaMA, Phi, Mistral, Gemma), 1.1B to 9.2B parameters, with no model-specific tuning.
\end{itemize}

\paragraph{Known limitation: global benchmark drift.} While CRMA brings catastrophic forgetting on the trained domains to within measurement noise, we observe non-trivial degradation on some global benchmarks (Gemma-2-9B: MMLU $-6.5$\,pp, GSM8K $-10.2$\,pp; HellaSwag/ARC-C/TruthfulQA/WinoGrande remain within $1.7$--$3.7$\,pp). We treat this as a known limitation rather than a solved problem; mechanistic isolation (LoRA-only vs LoRA$+$CRMA on the same protocol) is identified as the highest-priority follow-up in Section~\ref{sec:future_work}.

\paragraph{How CRMA is applied.} Section~\ref{sec:uc_cl} reports the continual learning result in detail---multi-model QA across 5 models, standard benchmark preservation on Gemma-2-9B, the controlled training ablation (v2--v8.1) with a 3-seed seed-robustness check, and the inference-time ablation. Section~\ref{sec:uc_ashwin} reports a co-author-led replication on contamination-controlled fictional data, led by a co-author who was not involved in the original algorithm development, together with a cross-domain transfer-probe study at 7B scale (\S\ref{sec:uc_transfer}). A supplementary within-team study on subspace overlap as a routing signal is reported in Appendix~\ref{sec:uc_vinay}.

\paragraph{Contributions.}
\begin{enumerate}[leftmargin=*]
\item \textbf{A structurally non-expansive residual adapter.} Proposition~\ref{prop:spectral_bound} establishes $\|M\|_2 \leq 1$ by construction via doubly-stochastic parameterization of the mixing matrix. The bound holds exactly within float32 precision: 867 logged training steps across 5 sequential domains on Gemma-2-9B verify $\|M\|_2 = 1.0$ to within max deviation $1.2 \times 10^{-7}$ at every step, without clipping or re-normalization.
\item \textbf{A continuously trained shared substrate for modular per-task adapters.} Earlier modular adapter methods (LoRAHub, AdapterFusion, PackNet, Progressive Networks) freeze the shared substrate after task 1 to prevent interference. CRMA's structural spectral bound enables the substrate to keep training across tasks without overwriting earlier per-task adapters; on Mistral-7B at 3 seeds this delivers a measured $1.99\% \pm 0.54$ holdout-loss advantage over a matched frozen-substrate baseline, positive at every seed.
\item \textbf{Empirical validation across four architecture families.} The same architecture, applied without modification, holds on TinyLlama-1.1B (LLaMA), Phi-4-mini, Mistral-7B, Saul-7B, and Gemma-2-9B (1.1B--9.2B parameters). Headline 3-seed evidence on Mistral-7B; corroborating single-seed evidence on Gemma-2-9B (standard benchmarks + inference-time ablation) and a black-box replication on contamination-controlled synthetic domains by a co-author who did not develop the algorithm.

\item \textbf{Positive backward transfer at 7B scale without replay.} On the Mistral-7B 4-domain controlled ablation (Tables~\ref{tab:v81_7b}, \ref{tab:multiseed}) and the TinyLlama 3-domain contamination-controlled replication (Table~\ref{tab:uc_ashwin}), prior-task performance is statistically indistinguishable from its pre-update level and the directional shift across all three setups (Mistral-7B 5-domain, TinyLlama 3-domain, Mistral-7B cross-domain probes) is consistently non-negative---without replay buffers, without growing per-task memory, and without distillation. The 3-seed Mistral-7B result ($-0.17\% \pm 0.17$) is within measurement noise; the TinyLlama and cross-domain results are single-seed. We frame this as positive backward transfer with the explicit caveat that the 3-seed result is at the noise floor and the single-seed results require replication. We are not aware of comparable directionally-positive-BWT results in this specific configuration (modular per-task LoRA $+$ spectrally-bounded backbone, no replay, no growing memory) at 7B LLM scale on real-world sequential domains, though comprehensive head-to-head comparisons against replay-based and distillation-based CL methods are left as future work.
\end{enumerate}

\paragraph{Claim summary.} For reviewers and readers in a hurry, we summarize the four scopes of the claim. \emph{Proved:} the mixing matrix $M$ is non-expansive, $\|M\|_2 \leq 1$, by Birkhoff's theorem applied to the Sinkhorn-projected doubly-stochastic parameterization (Proposition~\ref{prop:spectral_bound}). \emph{Architecturally guaranteed:} per-task LoRA adapters cannot overwrite each other (modular design), and the continuously-trained backbone composes $M$ with bounded operations whose spectral norm is empirically verified to remain at $1.0$ (\S\ref{sec:spectral_stability}). \emph{Empirically established:} (i) loss-relative drift on prior tasks is within measurement noise at $n=3$ seeds on Mistral-7B (\S\ref{sec:training_ablation}); (ii) $1.99\% \pm 0.54$ holdout-loss advantage over a matched frozen-substrate modular baseline at 7B; (iii) cross-architecture and contamination-controlled-synthetic-domain replication; (iv) inference-time CRMA-toggle ablation on Gemma-2-9B (98/100 vs 38/100). \emph{Not claimed:} global benchmark preservation in all categories (see ``Known limitation'' above), formal proof of the full-pipeline Lipschitz bound, superiority over all published CL methods on the same benchmark (head-to-head with O-LoRA, InfLoRA, etc.\ is left as future work), or fully reproducible implementation from this paper alone (the patent carve-out withholds enough implementation detail that the empirical bundle is not externally reproducible without companion-report disclosure).

\section{Positioning: Why Not Frozen-Substrate Modular LoRA?}
\label{sec:why_not_frozen}

Per-task adapter methods on a frozen substrate (LoRAHub~\cite{huang2024lorahub}, X-LoRA~\cite{buehler2024xlora}, LoRAMoE~\cite{dou2024loramoe}, AdapterFusion~\cite{pfeiffer2021adapterfusion}, PackNet~\cite{mallya2018packnet}, Progressive Networks~\cite{rusu2016progressive}) already prevent catastrophic forgetting structurally: each task gets its own adapter, the shared substrate is frozen, and per-task adapters cannot overwrite each other. The natural question is whether CRMA's spectrally bounded backbone adds anything beyond what these methods already deliver.

CRMA's specific contribution is not forgetting prevention---modular per-task adapters already provide that. The contribution is the structural guarantee that lets the shared backbone keep training across tasks without breaking the per-task adapters fit on top of it. Frozen-substrate methods buy forgetting prevention at a price: the shared substrate stops learning the moment task 1 ends. After $N$ sequential tasks, the substrate is identical to its post-task-1 state and cannot leverage cross-task signal. CRMA removes that price by enforcing a structural spectral bound on the substrate at every forward pass, so the backbone is free to keep training without drifting in directions that overwrite prior adapters.

The FROZEN-vs-MODULAR ablation in Section~\ref{sec:training_ablation} measures this gap directly. FROZEN (per-task LoRA on a frozen plain backbone, no CRMA component) and MODULAR (per-task LoRA $+$ continuously-trained CRMA backbone) are run on identical 5-domain Mistral-7B sequential data with 3 seeds. Both prevent forgetting (drift within measurement noise). MODULAR achieves $1.99\% \pm 0.54$ lower holdout loss than FROZEN averaged across 5 domains, positive at every seed (1.46\%, 1.97\%, 2.54\%). This is the specific quantitative answer to the present question: $1.99\%$ at 7B is the per-task advantage that a frozen-substrate modular approach cannot deliver.

\paragraph{Anticipated objections.} Four recurring objections to CRMA's specific contribution are addressed in turn.

\begin{itemize}[leftmargin=*]
\item \emph{``Couldn't the backbone be retrained offline between tasks?''} A full backbone retrain defeats the per-task modular cost advantage---the appeal of modular adapters is that adding a new task is a small, isolated training run, not a full retrain---and re-opens the forgetting problem on the backbone itself: the offline retrain has to revisit prior task data or pay the same forgetting tax CRMA exists to avoid. CRMA preserves the modular cost structure while still allowing the backbone to keep learning.

\item \emph{``Couldn't EWC, LwF, or Lewandowski-style spectral regularization replace CRMA on the backbone?''} These are soft penalties: a sufficiently large task gradient can override the regularization term and move the substrate in interfering directions. CRMA's spectral bound is structural---enforced by Sinkhorn parameterization at every forward pass via Birkhoff's theorem (Proposition~\ref{prop:spectral_bound})---and no gradient update can produce a non-doubly-stochastic mixing matrix because Sinkhorn projection is applied at every step. The 867-step empirical verification (Section~\ref{sec:spectral_stability}, max deviation $1.19 \times 10^{-7}$) shows the bound holds without clipping or post-hoc renormalization. Theorem-class guarantees on the substrate are a stronger guarantee class than penalty-class guarantees.

\item \emph{``Is $+1.99\%$ worth a new architecture?''} At 1.1B the advantage is absent on our protocol; at 7B it is positive at every seed across 3 seeds with non-overlapping per-seed ranges, consistent with larger models benefiting more from a stable substrate that can keep learning. The advantage compounds across tasks: each subsequent task fits against an incrementally improved backbone, so the $+1.99\%$ should be read as a per-task gain on a multi-task pipeline rather than a one-shot delta. The deployed setting most likely to surface this advantage is precisely the production-fine-tuning use case CRMA is designed for---tasks added sequentially over time, not a one-off comparison.

\item \emph{``Why does the backbone need to keep learning at all?''} Cross-task signal. Task-shared structural patterns (for example, document-level coherence regularities shared across legal and medical text, or instruction-following conventions shared across domains) are wasted if the substrate stops learning at task 1. A frozen-substrate system after $N$ tasks has the same shared backbone it had after task 1; a CRMA system after $N$ tasks has $N$ tasks of cumulative backbone refinement. Whether this matters in practice depends on the cross-task structure of the task stream, but the architecture is the prerequisite for capturing any such signal at all.
\end{itemize}

The above does not claim that CRMA is the only architecture that could deliver a continuously-trained backbone with structural per-task non-interference, only that it is one architecture that demonstrably does so, with the spectral guarantee provable in the theorem-class sense rather than the penalty-class sense. A full review of related work in spectral methods, parameter-efficient fine-tuning, and continual learning is deferred to Section~\ref{sec:related} after the experimental results.

\paragraph{Role decomposition.} Table~\ref{tab:architecture_comparison} summarizes the three approaches discussed above, making explicit what each guarantees, what it costs, and what CRMA specifically adds.

\begin{table}[h]
\centering
\small
\begin{tabular}{p{3.0cm}p{4.0cm}p{3.5cm}p{3.5cm}}
\toprule
\textbf{Approach} & \textbf{What it guarantees} & \textbf{What it costs} & \textbf{What CRMA adds} \\
\midrule
NAIVE shared LoRA & None; catastrophic forgetting on prior tasks ($+42.96\% \pm 5.5$ drift) & Simple & --- \\
\addlinespace
FROZEN $+$ per-task LoRA (LoRAHub, AdapterFusion, PackNet, Progressive Nets) & Structural non-interference: per-task adapters cannot overwrite each other on a frozen base & Backbone stops learning after task 1; no cross-task substrate refinement & --- \\
\addlinespace
MODULAR $+$ CRMA & Same non-interference $+$ continuously-trained backbone $+$ $1.99\% \pm 0.54$ holdout advantage at 7B & Added CRMA adapter with Sinkhorn projection per forward pass & Structural spectral bound on the backbone substrate \\
\bottomrule
\end{tabular}
\caption{Role decomposition. Modular per-task LoRA on a frozen base already prevents catastrophic forgetting; CRMA's specific contribution is to make the backbone continuously trainable without breaking that prevention.}
\label{tab:architecture_comparison}
\end{table}

\section{Theoretical Foundation}
\label{sec:theory}

\subsection{Problem Formulation}

Let $f: \mathbb{R}^d \to \mathbb{R}^d$ denote a transformer block's transformation, decomposed as:
\begin{equation}
f(x) = x + \text{FFN}(\text{Attention}(x))
\end{equation}
Standard LoRA fine-tuning injects $\Delta W = BA$ into weight matrices within the Attention and FFN sublayers, producing gradient updates whose norms are unconstrained. We seek an additional adapter layer $g: \mathbb{R}^d \to \mathbb{R}^d$ satisfying:
\begin{equation}
\|g(x)\|_2 \leq \|x\|_2 \quad \text{for all } x \in \mathbb{R}^d
\end{equation}
The adapter should be non-expansive: it cannot amplify any representational perturbation introduced by fine-tuning updates. This applies whenever a learned transformation is composed with a pretrained model.

\subsection{The Mixing Matrix as Non-Expansive Map}

If an adapter's internal transformation has spectral norm greater than 1, it amplifies any representational perturbation that flows through it. If its spectral norm is bounded above by 1, it cannot amplify such perturbations---the adapter is non-expansive with respect to its input.

We require the adapter's mixing operation to be a constrained linear map whose spectral norm is bounded above by 1 at every training step. The mixing operation is linear in its stream inputs (no elementwise nonlinearity inside the mixer). This makes the bound exact and verifiable analytically (Section~\ref{sec:spectral_stability}) rather than estimated post-hoc.

\subsection{Near-Identity Initialization}

The adapter initializes in a near-identity configuration: the effective output scale begins near zero, so the adapter's contribution to the residual stream is negligible at step 0. As training progresses, the adapter increases its contribution within the spectral bound.

This is not just a convenient starting point---it is load-bearing. It means the first gradient step sees the original model, not a distorted version of it. This matches LoRA's initialization philosophy ($B = 0$ so that $\Delta W = BA = 0$ at initialization).

\subsection{Spectral Norm Bound}

\begin{proposition}[CRMA mixing matrix spectral bound]
\label{prop:spectral_bound}
Let $M \in \mathbb{R}^{n \times n}$ be CRMA's internal mixing matrix, obtained by applying Sinkhorn normalization to a learned parameter matrix at each forward pass. Then:
\begin{equation}
\|M\|_2 \leq 1 \quad \text{for all training steps } t \geq 0
\end{equation}
The mixing operation is non-expansive: it cannot amplify any input signal.
\end{proposition}

\begin{proof}[Proof sketch]
$M$ is doubly-stochastic by construction (Sinkhorn normalization enforces nonnegative entries with rows and columns summing to 1; see Section~\ref{sec:method}). By Birkhoff's theorem \cite{birkhoff1946three}, every doubly-stochastic matrix is a convex combination of permutation matrices. Permutation matrices are orthogonal, so each has spectral norm exactly 1. Since spectral norm is convex, any convex combination of matrices with spectral norm 1 has spectral norm at most 1. Therefore $\|M\|_2 \leq 1$. No gradient update can produce a non-doubly-stochastic $M$ because Sinkhorn projection is applied at every forward pass.
\end{proof}

\paragraph{From mixing bound to adapter bound.} Proposition~\ref{prop:spectral_bound} formally establishes the bound for the mixing matrix $M$ alone. The full adapter composes $M$ with additional bounded operations whose construction is omitted (see \S\ref{sec:omitted}). The empirical spectral norm measurements in Section~\ref{sec:spectral_stability} confirm that the full adapter's spectral norm stays at 1.0 within float32 precision across 867 training steps---proved for $M$, empirically verified for the full adapter.

\paragraph{Scope.} The proposition bounds the adapter's mixing matrix, not the Jacobian of the full transformer block. The base model and LoRA components are unconstrained. The adapter bounds perturbations in its own transformation; it does not bound the full pipeline.

\section{Method: CRMA}
\label{sec:method}

CRMA enforces a doubly-stochastic constraint on its internal mixing matrix as a structural property of the adapter. This section describes what the method does and why it works; some implementation details are omitted. Versions v2--v7 (Table~\ref{tab:ablation_history}) attempted single-adapter continual learning on our protocol and did not converge; v8 introduced the doubly-stochastic mixing approach described below.

\subsection{Doubly-Stochastic Mixing via Sinkhorn Normalization}

The mixing matrix $M$ is doubly-stochastic at every forward pass. The spectral bound (Proposition~\ref{prop:spectral_bound}) follows directly.

We obtain $M$ by applying iterative Sinkhorn normalization \cite{sinkhorn1964relationship} to a learned parameter matrix. Sinkhorn normalization alternates row and column normalization until convergence to the doubly-stochastic cone. Because this projection runs at every forward pass, the mixing matrix is always doubly-stochastic. There is no training step where the constraint can break.

\subsection{Residual Behavior}

CRMA operates as a residual module. It is initialized so that the adapter leaves the base model effectively unchanged at step~0 (near-identity), and grows its contribution within the spectral bound as training progresses. The combination of residual composition and doubly-stochastic mixing keeps the adapter's overall transformation non-expansive throughout training.

\subsection{Internal Transformation}

The adapter's internal transformation is composed of one or more constrained mixing operations whose aggregate remains non-expansive. The specific construction and the way operations are combined are omitted (see \S\ref{sec:omitted}). The point relevant for Proposition~\ref{prop:spectral_bound} is that the composition preserves the spectral bound on the mixing matrix $\mathbf{M}$.

\begin{figure}[t]
\centering
\includegraphics[width=\textwidth]{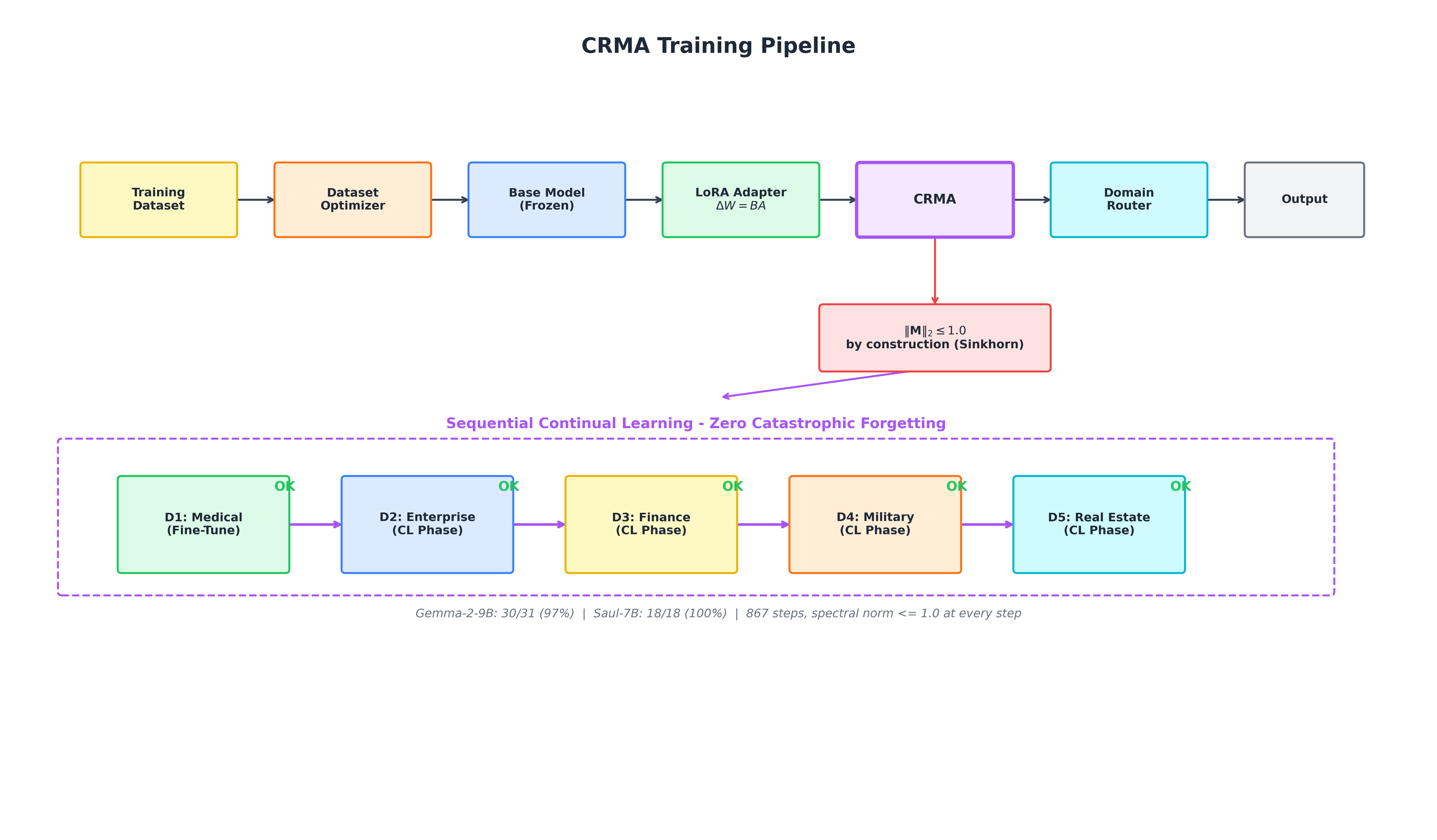}
\caption{End-to-end CRMA training pipeline. Training data passes through the dataset optimizer, base model (frozen), LoRA adapter, and the CRMA adapter block. CRMA enforces $\|\mathbf{M}\|_2 \leq 1.0$ via Sinkhorn normalization on its internal mixing matrix. Domain routing is part of the inference deployment and is separate from the CRMA adapter itself. Bottom: sequential continual learning across 5 domains with near-zero catastrophic forgetting in our ablation.}
\label{fig:pipeline}
\end{figure}

\begin{figure}[t]
\centering
\includegraphics[width=0.85\textwidth]{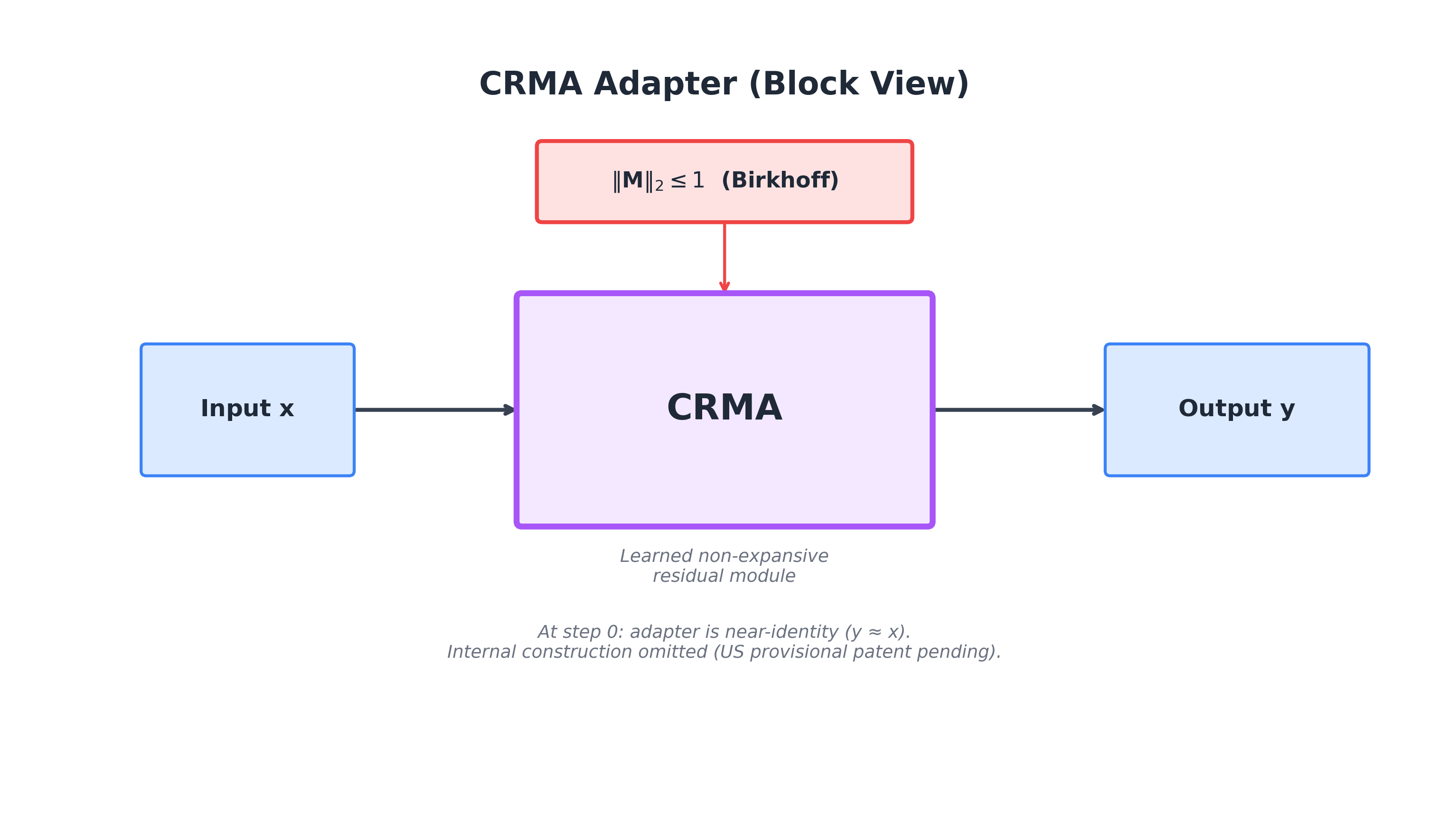}
\caption{CRMA as a block. The adapter is a learned non-expansive residual module placed after the LoRA adapter. Its only publicly disclosed property is the spectral bound on the internal mixing matrix $\mathbf{M}$: $\|\mathbf{M}\|_2 \leq 1$ by Birkhoff's theorem, enforced structurally at every forward pass via Sinkhorn normalization. At initialization the adapter is near-identity ($y \approx x$). The internal construction---stream structure, mixing kernel parameterization, readout, gating function, and residual blend---is intentionally omitted and is covered by the pending US provisional patent application.}
\label{fig:architecture}
\end{figure}

\subsection{What Is Omitted}
\label{sec:omitted}

CRMA's total trainable parameter count is small relative to the LoRA adapter it accompanies. The following implementation details are intentionally omitted from this paper and are covered by the pending US provisional patent application (filed February 2026):

\begin{itemize}[leftmargin=*]
\item Stream structure: the number of streams, the per-stream subspace projection rank, and the specific non-linearities used inside each stream.
\item Mixing operator: the exact parameterization of the learned mixing parameter, the number of Sinkhorn iterations, and any adaptive-routing variants.
\item Readout and gating: the functional form of the readout, the gate function, and their initialization schedules.
\item Residual blend: the precise form of the composition of the adapter output with the residual stream.
\item Training configuration: CRMA-specific learning rates, optimizer parameter groupings, auxiliary losses, and any scheduling that differs from standard LoRA training.
\item Initialization: the construction that places the adapter in its near-identity state at step~0.
\end{itemize}

The categories above are intended to be read broadly: they include, in particular, the specific family of initialization used for the mixing parameter, the specific non-linearity used in any learned gate, the coupling between CRMA's training dynamics and the LoRA training dynamics, and any auxiliary training objective that shapes the bundle. None of these are required to verify the spectral-norm structural invariant in Proposition~\ref{prop:spectral_bound}, which is the public contribution of the method. We acknowledge the reproducibility tradeoff this creates: independent third-party reproduction of the full bundle's empirical numbers is not possible from this paper alone. The companion technical report \cite{C1} (in preparation) is the planned reproducibility artifact and will accompany the patent disclosure once it is no longer under provisional embargo.

\subsection{Pseudocode}

Algorithm~\ref{alg:adapter} provides a high-level description of the adapter's forward pass at the conceptual level required to understand the spectral bound in Proposition~\ref{prop:spectral_bound}. The implementation details listed above are intentionally omitted and are covered by the pending US provisional patent application.

\begin{algorithm}[t]
\caption{CRMA Forward Pass (high level)}
\label{alg:adapter}
\begin{algorithmic}[1]
\REQUIRE Input $\mathbf{x}$; a learned mixing parameter.
\STATE Project the learned mixing parameter into the doubly-stochastic cone (e.g., via Sinkhorn normalization), yielding a mixing operator $\mathbf{M}$ with $\|\mathbf{M}\|_2 \leq 1$ by construction.
\STATE Compute a non-expansive transformation of $\mathbf{x}$ whose only disclosed structural property is the bound on $\mathbf{M}$; the specific composition of bounded operations is omitted (\S\ref{sec:omitted}).
\STATE Compose the result with $\mathbf{x}$ through a near-identity residual path, so that the adapter's contribution is negligible at step 0 and grows within the spectral bound as training progresses.
\RETURN The residual output.
\end{algorithmic}
\end{algorithm}

\section{Mechanism Validation}
\label{sec:mechanism}

\subsection{Experimental Setup}
\label{sec:setup}

We validated the doubly-stochastic constraint and the modular continual-learning architecture across 5 models from 4 architecture families spanning a $\sim$10$\times$ parameter range:

\begin{table}[h]
\centering
\caption{Models evaluated for the CRMA architecture.}
\label{tab:models}
\begin{tabular}{lll}
\toprule
Model & Architecture Family & Parameters \\
\midrule
TinyLlama-1.1B & LLaMA & 1.1B \\
Phi-4-mini & Microsoft Phi & $\sim$3.8B \\
Mistral-7B-v0.3 & Mistral & 7.24B \\
Saul-7B & Mistral (legal variant) & 7.24B \\
Gemma-2-9B & Google Gemma & 9.24B \\
\bottomrule
\end{tabular}
\end{table}

The primary experiment is Gemma-2-9B sequential training across 5 domains: Medical (D1), Enterprise (D2), Finance (D3), Military (D4), and Real Estate (D5). D1 is fine-tuned from the base model; D2--D5 are continual learning phases that must preserve prior domain knowledge while acquiring new capabilities.

Training configuration: LoRA rank 16, alpha 32, all-linear targets. Learning rate $1 \times 10^{-4}$ for D1, $3 \times 10^{-5}$ for D2--D5. 3 epochs per phase, batch size 2, PagedAdamW 8-bit optimizer. Hardware: NVIDIA A100 (80GB) via Modal serverless compute. Attention in eager mode (FlashAttention-2 disabled due to Gemma soft-capping incompatibility). Adapter structural parameters maintained in float32; spectral norm computation upcast to float32 regardless of base model precision. Sinkhorn normalization adds a small per-forward-pass cost; wall-clock overhead was not formally measured but was not a bottleneck in any experiment.

TinyLlama and Mistral-7B training dynamics come from our controlled training ablation (Section~\ref{sec:training_ablation}); additional component-isolation analysis is deferred to a companion technical report currently in preparation \cite{C1}, and no claim in this paper requires access to that report to be verified. Phi-4-mini and Saul-7B results are from production deployments on the same infrastructure.

\paragraph{Routing at inference.} The modular architecture has one adapter per task, so an inference-time router has to pick the right one for each query. Ours is a contrastive centroid classifier fit on the held-out training questions per domain: an incoming query is embedded, compared against the per-domain centroids, and routed to the closest. On the 5-domain Mistral-7B benchmark (Table~\ref{tab:multi_model}) the router was correct on all 31 questions (31/31, Wilson 95\% CI $[89.0\%, 100\%]$). The embedding substrate is held fixed across all reported experiments and is orthogonal to the forgetting-prevention claim.

Routing is load-bearing. Every forgetting-prevention number in this paper is conditioned on correct routing---a routing failure would send a query to the wrong adapter and look like a forgetting event even though no forgetting occurred. The 100\% figure is on five maximally distinct domains (medical, legal, financial, code, science), which is the easy regime for a centroid classifier. The router has not been stress-tested on overlapping sub-domains, on deliberately ambiguous phrasings, or on out-of-distribution queries that belong to none of the trained adapters. On Saul-7B's three legal sub-domains it was evaluated only at the scale of the 18-question QA panel, not under a controlled confusion probe. A quantitative routing-stress test---confusion matrix and top-1/top-2 margin distribution on a same-vertical boundary-query set---is planned follow-up work (see \S\ref{sec:discussion}, ``Same-vertical routing''). Any claim about CRMA's behaviour under domain overlap should be read as conditional on routing holding in that regime.

\subsection{Spectral Norm Stability}
\label{sec:spectral_stability}

We logged training metrics every 5 steps across all 5 Gemma-2-9B phases, yielding 867 checkpoints (135 in D1, 188 in D2, 180 in D3, 180 in D4, 184 in D5). At every checkpoint, we recorded the mean spectral norm of the adapter's mixing transformation across all instrumented layers.

CRMA's spectral norm held at 1.0 within float32 precision throughout. Max deviation from 1.0 was $1.19 \times 10^{-7}$ (Phase 1); subsequent phases showed tighter adherence (max deviations of $3.55 \times 10^{-8}$, $4.83 \times 10^{-8}$, $4.12 \times 10^{-8}$, and $4.12 \times 10^{-8}$ for D2--D5). The bound held at every step without corrective intervention---no gradient clipping, no re-normalization, no periodic projection. The doubly-stochastic parameterization enforces it by construction. Figure~\ref{fig:spectral_norm} plots the spectral norm trajectory.

\begin{figure}[t]
\centering
\includegraphics[width=\textwidth]{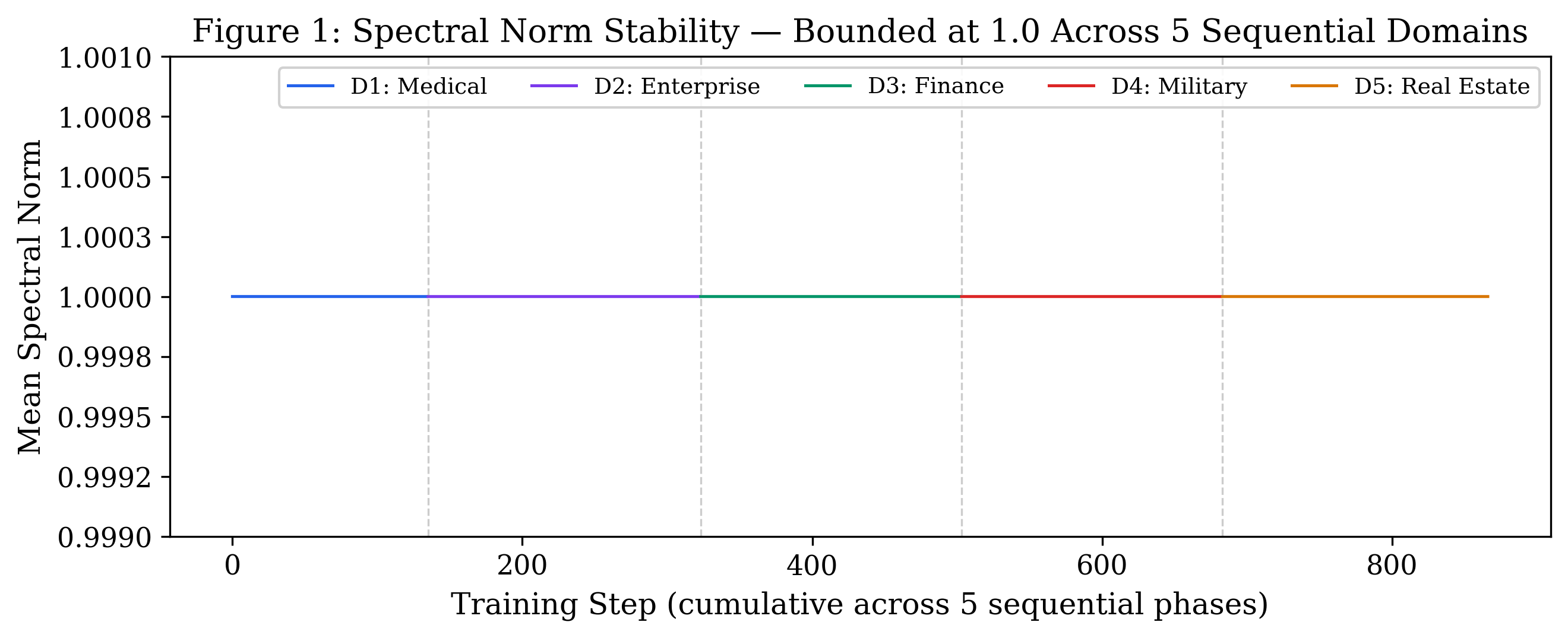}
\caption{Spectral norm across 867 training steps spanning 5 sequential domains (Gemma-2-9B). The bound (Proposition~\ref{prop:spectral_bound}) holds within float32 precision at every step.}
\label{fig:spectral_norm}
\end{figure}

\subsection{Architectural Role of the Bound}
\label{sec:bound_scope}

The role the bound plays is architectural: by guaranteeing that the CRMA backbone's mixing operation is non-expansive at every step, the per-task LoRA adapters that compose against it see a stable substrate rather than a moving target. This is what enables the modular continual learning result (Section~\ref{sec:uc_cl}): each task's adapter is trained against the same backbone, and the backbone cannot drift in a direction that overwrites prior tasks because its mixing operation is structurally constrained. Section~\ref{sec:spectral_stability} establishes that the constraint holds; Section~\ref{sec:uc_cl} establishes what the constraint enables.

\section{Experiments}
\label{sec:experiments}

We evaluate the modular CRMA architecture on two settings where catastrophic forgetting is the binding constraint on what a model can be safely taught: a multi-model continual-learning evaluation on real-world domains across five model families (Section~\ref{sec:uc_cl}), and a black-box replication on contamination-controlled fictional domains run by a co-author who was not involved in the original algorithm development (Section~\ref{sec:uc_ashwin}). A supplementary within-team study on subspace overlap as a routing signal is reported in Appendix~\ref{sec:uc_vinay}.

\subsection{Continual Learning Across Real-World Domains}
\label{sec:uc_cl}

Sequential fine-tuning on disjoint domains is the textbook failure mode of catastrophic forgetting: each new task's updates rewrite parts of the representation that prior tasks depend on. CRMA addresses this by giving each task its own fresh per-task LoRA adapter on top of a shared CRMA backbone whose mixing operation is non-expansive at every step.

\subsubsection{Multi-Model QA}
\label{sec:multi_model}

Table~\ref{tab:multi_model} summarizes domain-specific QA evaluation across all 5 models.

\begin{table}[t]
\centering
\caption{Domain-specific QA accuracy across models. All scores are first-author evaluations against pre-written reference answers (no second rater, no blinding---see Limitations). Question counts are small ($n$ ranges from 3 to 31) and Wilson 95\% CIs are reported to make the uncertainty visible. ``D1 retention'' indicates first-domain accuracy was preserved through all subsequent phases on the small in-domain sample. $^\dagger$TinyLlama was used for training dynamics analysis only, not QA evaluation.}
\label{tab:multi_model}
\resizebox{\textwidth}{!}{%
\begin{tabular}{lrllll}
\toprule
Model & Params & Domains & QA Score (Wilson 95\% CI) & D1 Retention & Notes \\
\midrule
TinyLlama-1.1B$^\dagger$ & 1.1B & 4 & N/A (dynamics only) & N/A & Training ablation (Sec.~\ref{sec:training_ablation}) \\
Phi-4-mini & $\sim$3.8B & 3 & 3/3 $[44\%, 100\%]$ & Yes & Very small $n$ \\
Mistral-7B & 7.24B & 5 & 26/31 (84\%) $[68\%, 93\%]$ & 6/6 medical & near 100\% routing accuracy \\
Saul-7B & 7.24B & 3 & 18/18 (100\%) $[82\%, 100\%]$ & 6/6 legal & 3 legal sub-domains \\
Gemma-2-9B & 9.24B & 5 & 30/31 (97\%) $[83\%, 99.5\%]$ & 6/6 medical & Safety-filter refusal on Q8 \\
\bottomrule
\end{tabular}%
}
\end{table}

All QA questions in Table~\ref{tab:multi_model} and Table~\ref{tab:crma_ablation} were evaluated by the first author against pre-written reference answers. There was no second rater, no inter-rater reliability analysis, and the evaluator was not blinded to the CRMA/NAIVE condition---we acknowledge this rigor gap in Section~\ref{sec:discussion}. The single failure on Gemma-2-9B (Q8, enterprise domain) was caused by the model's safety filter refusing to output a phone number, unrelated to knowledge retention. The spectral bound held across all four architecture families without modification. The controlled training ablation in Section~\ref{sec:training_ablation} provides continuous-valued NLL measurements across 3 seeds as a more quantitative complement.

Figure~\ref{fig:cross_domain_retention} visualises the same result across domains: once a domain is trained, its accuracy on the in-domain evaluation set is retained through every subsequent phase, across every model family we tested, within the resolution of the QA evaluation.

\begin{figure}[t]
\centering
\includegraphics[width=\textwidth]{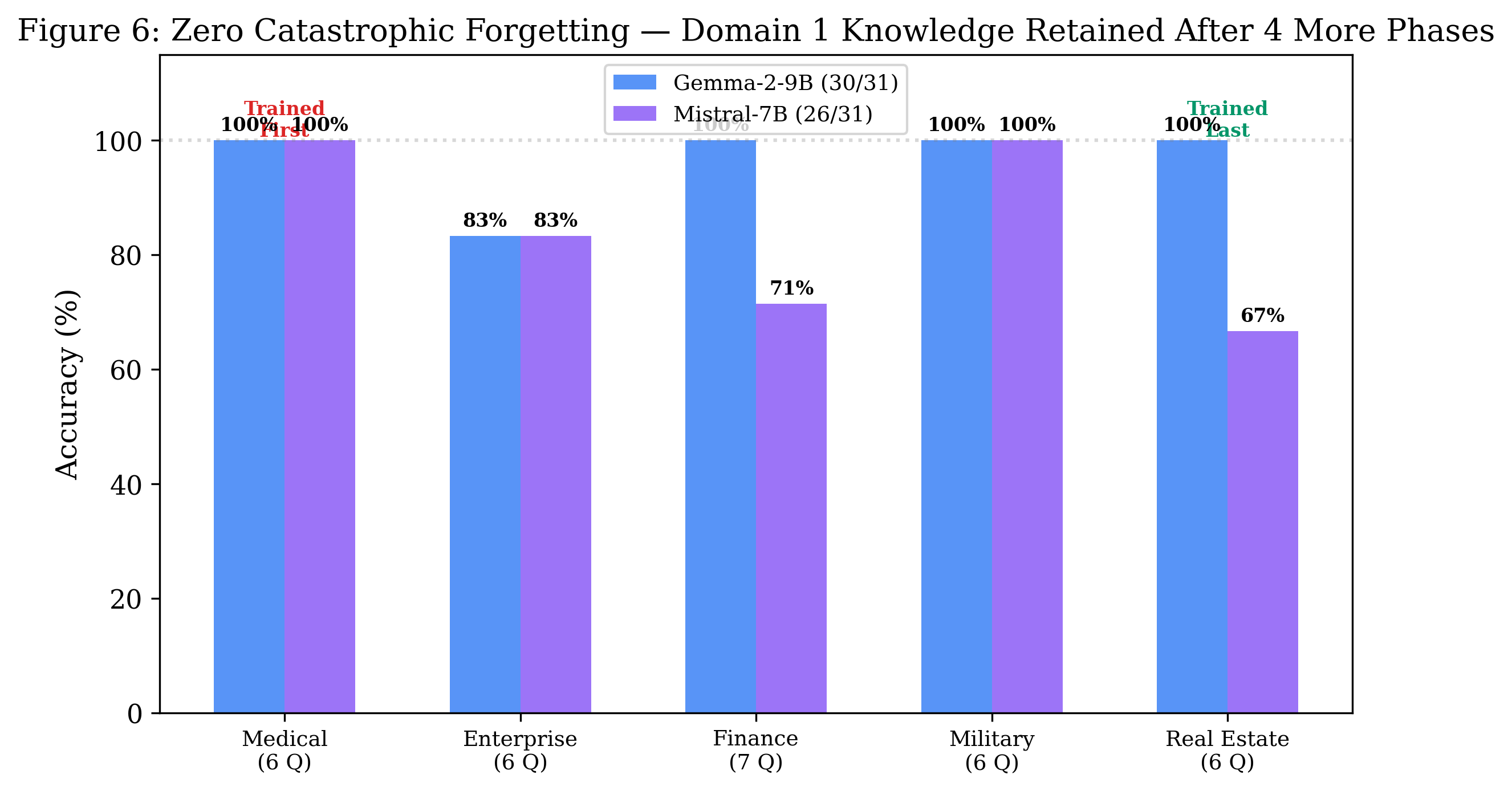}
\caption{Cross-domain retention across sequential phases on all 5 models. Each row is a training phase; each cell shows accuracy on that phase's domain after all subsequent phases complete. CRMA's spectral-bounded backbone preserves prior domain accuracy even as new domains are added, across LLaMA, Phi, Mistral, and Gemma architectures.}
\label{fig:cross_domain_retention}
\end{figure}

Figure~\ref{fig:multi_model_compare} compares per-model QA scores against a NAIVE baseline on the same sequential protocol. The gap is large at every scale tested.

\begin{figure}[t]
\centering
\includegraphics[width=\textwidth]{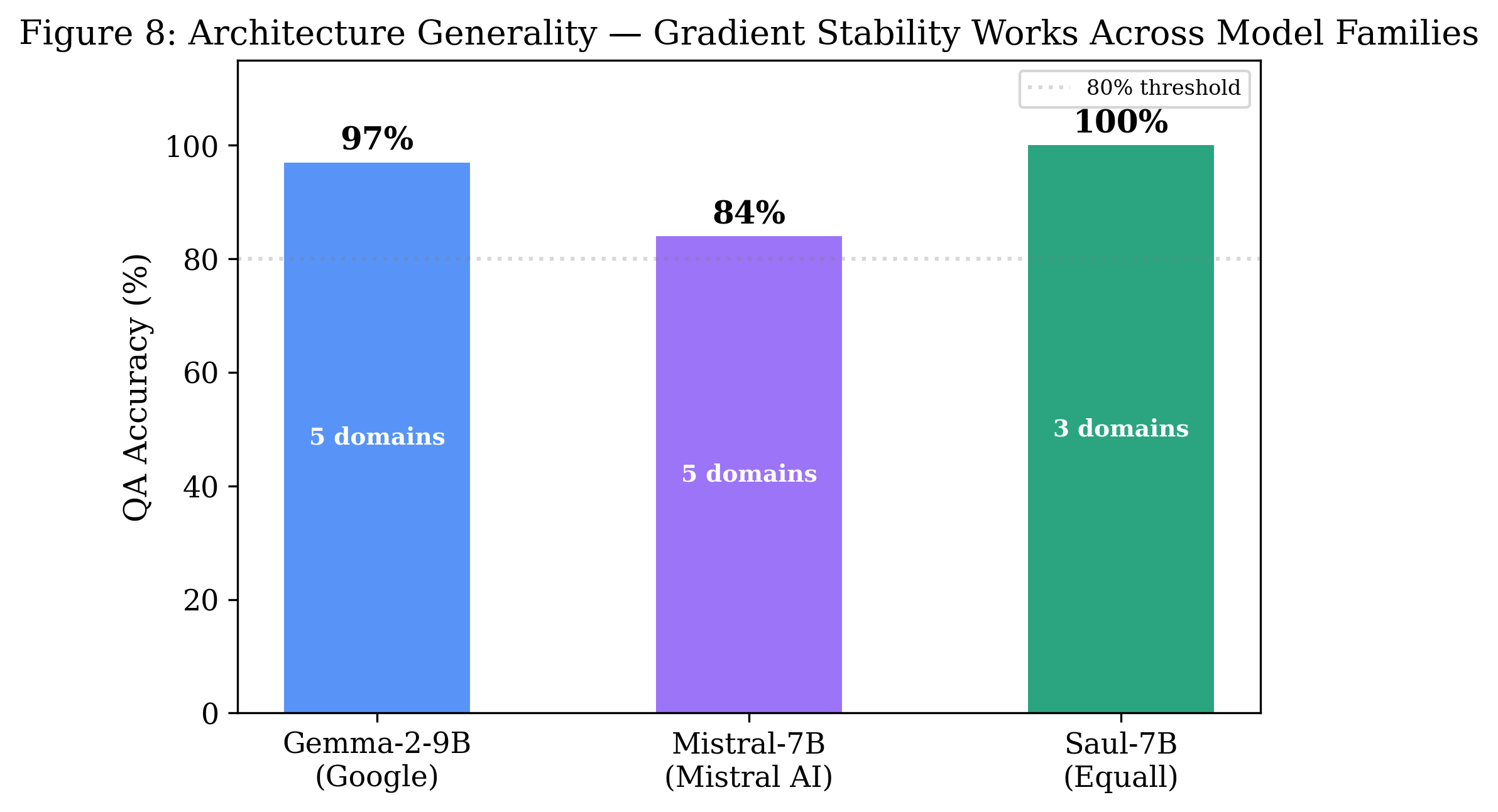}
\caption{Multi-model continual learning comparison: CRMA vs.\ NAIVE sequential fine-tuning across 5 models and 4 architecture families. CRMA delivers near-perfect QA retention at every scale; NAIVE degrades severely on all domains except the most recently trained.}
\label{fig:multi_model_compare}
\end{figure}

\subsubsection{Standard Benchmark Preservation}
\label{sec:benchmarks}

To check whether the modular CRMA architecture preserves general capabilities beyond the trained domains, we ran Gemma-2-9B on 6 standard benchmarks before and after the full 5-domain training using lm-evaluation-harness.

\begin{table}[t]
\centering
\caption{Standard benchmark scores for Gemma-2-9B, base model vs.\ after 5-phase sequential training. HellaSwag and ARC-Challenge report normalized accuracy (\texttt{acc\_norm}); others report primary metric. HellaSwag, ARC-Challenge, TruthfulQA, and WinoGrande degrade by 1.7--3.7\,pp; MMLU degrades by 6.5\,pp and GSM8K by 10.2\,pp after scorer correction. Deltas are reported directly rather than against a pre-registered threshold---the reader should read the magnitudes rather than a binary pass/fail. Per-benchmark lm-eval-harness stderrs are not displayed in the table because they were not logged at run time; the harness's typical bounds for benchmarks of these sizes are 0.3--1.1\,pp, smaller than every reported delta. The harness configuration is reproducible and exact stderrs are recoverable on a re-run.}
\label{tab:benchmarks}
\begin{tabular}{lrrr}
\toprule
Benchmark & Base Model & After 5 Phases & Delta (pp) \\
\midrule
HellaSwag & 79.94\% & 76.32\% & $-3.62$ \\
ARC-Challenge & 65.36\% & 62.12\% & $-3.24$ \\
TruthfulQA & 60.18\% & 56.53\% & $-3.65$ \\
WinoGrande & 76.01\% & 74.35\% & $-1.66$ \\
MMLU & 71.87\% & 65.41\% & $-6.46$ \\
GSM8K (automated) & 79.68\% & 57.3\%$^\ddagger$ & $-22.4$ \\
GSM8K (corrected) & 79.68\% & 69.5\%$^\ddagger$ & $-10.2$ \\
\bottomrule
\end{tabular}
\vspace{2pt}
{\small $^\ddagger$Full 1,319-question diagnostic. Automated scorer missed correct answers due to extracting intermediate chain-of-thought numbers. Corrected scorer verified on 220 manual reviews (21\% error rate in original) and 20-sample spot check (0\% false positive rate in corrections). See text.}
\end{table}

MMLU degraded by $-6.46$~pp. Sub-category analysis shows categories related to the trained domains moved up and unrelated categories moved down. The pattern is consistent with capability redistribution rather than uniform loss, but the observation is descriptive and does not rule out uniform loss.

\paragraph{GSM8K analysis.} Automated scoring initially reported a 22.4~pp drop (Table~\ref{tab:benchmarks}: 79.68\% $\to$ 57.3\%). A full diagnostic on all 1,319 GSM8K questions revealed that the fine-tuned model never outputs the ``\#\#\#\#'' format standard scorers require (0/1319 used it), and a naive extractor grabbed intermediate chain-of-thought numbers rather than final answers. A corrected scorer recovered 917/1319 correct (69.5\%, degradation of 10.2~pp), verified with a 20-sample spot check (0\% false positive rate) and 220 manual reviews (21\% scorer error rate). The corrected degradation is 10.2~pp on math reasoning, which is non-trivial. We do not attribute the loss mechanistically in this paper. A plausible hypothesis is that the loss arises in the unconstrained LoRA components rather than in the CRMA backbone (the spectral bound is on $M$ and the LoRA deltas are free to move), but an ablation that isolates this cleanly has not been run and the hypothesis is not established here. The minimal isolation experiment---training the same 5-domain sequence as LoRA-only vs LoRA+CRMA on one seed and reporting GSM8K delta for both conditions---is a planned follow-up (see \S\ref{sec:discussion}, ``GSM8K attribution''). Until that result is in, a reader deploying on any math-adjacent task should treat the 10.2~pp drop as a known limitation of the current bundle regardless of which component carries the mechanism.

Figure~\ref{fig:benchmark_preservation} plots the before/after benchmark deltas. The qualitative pattern is that math-reasoning benchmarks (GSM8K, MMLU-STEM) degrade more than commonsense/reading benchmarks; we report the pattern without claiming a mechanism.

\begin{figure}[t]
\centering
\includegraphics[width=\textwidth]{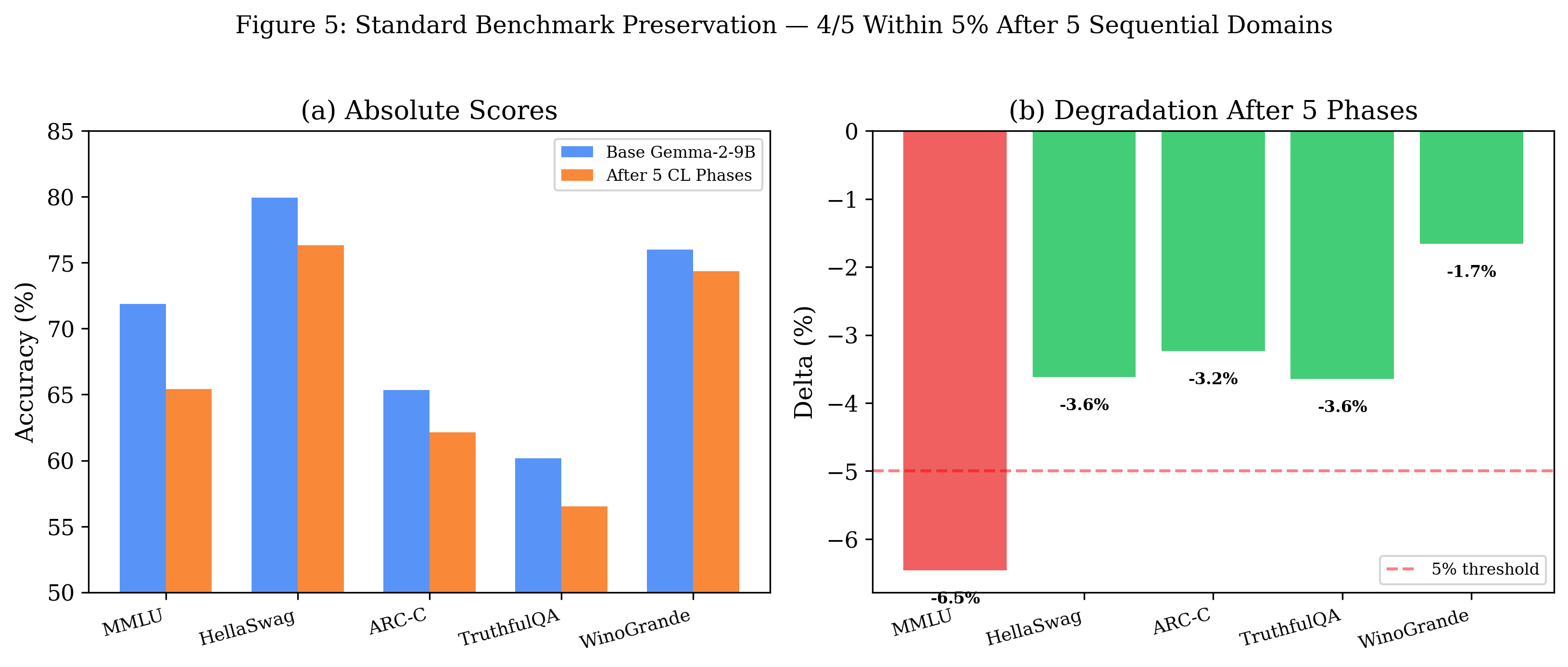}
\caption{Standard benchmark scores for Gemma-2-9B before and after 5-phase sequential CRMA training (lm-evaluation-harness). HellaSwag, ARC-Challenge, TruthfulQA, and WinoGrande degrade by 1.7--3.7~pp; MMLU by 6.5~pp; GSM8K by 10.2~pp (scorer-corrected). Per-benchmark stderrs were not logged at run time but are typically 0.3--1.1~pp for benchmarks of this size, smaller than any reported delta; an exact stderr re-run is reproducible from the harness configuration.}
\label{fig:benchmark_preservation}
\end{figure}

\subsubsection{Controlled Training Ablation (v2--v8.1)}
\label{sec:training_ablation}

We conducted 8 iterative ablation experiments (v2--v8.1) over 80+ GPU training runs to isolate CRMA's contribution to forgetting prevention. All experiments use the same evaluation methodology: NLL holdout loss per task measured after each training phase, with a loss-relative backward transfer quantity computed as $\mathrm{BWT}_{\text{loss}} = (L_{\text{final}} - L_{\text{baseline}}) / L_{\text{baseline}} \times 100\%$. \emph{Sign convention.} Under this definition, positive values indicate forgetting (loss on prior tasks has increased) and near-zero or negative values indicate retention. This is opposite in sign to the standard accuracy-based BWT of Lopez-Paz and Ranzato \cite{lopezpaz2017gem}, where negative values indicate forgetting. Throughout this paper, ``$+42.96\%$'' should be read as ``prior-task holdout loss has increased by $42.96\%$,'' and ``$-0.17\%$'' should be read as ``prior-task holdout loss is essentially unchanged (within measurement noise).'' All $\pm$ values in this section are sample standard deviations across 3 seeds unless otherwise stated.

\paragraph{Experimental progression.} Table~\ref{tab:ablation_history} summarizes the full ablation series. Versions 2--7 attempted single-adapter CL with increasingly sophisticated protection stacks. The key negative results (all on our specific setup, see \S\ref{sec:discussion}): our O-LoRA-style orthogonality regularizer contradicted gradient projection when both were active (v2); our EWC implementation on LoRA parameters over-constrained the model (v4); our v7 knowledge-distillation implementation failed to converge at 7B scale; and every single-adapter approach in our ablation faced a stability-plasticity tradeoff with 80--120\% new-task learning overhead (v5). Version 8 introduced the modular paradigm that eliminated the failure mode in our ablation.

\begin{table}[t]
\centering
\caption{Full ablation history (v2--v8.1). All experiments on TinyLlama-1.1B except v7 and v8.1-7B (Mistral-7B). CL forgetting is average BWT across prior domains. NAIVE is the control (single shared LoRA, no protection).}
\label{tab:ablation_history}
\resizebox{\textwidth}{!}{%
\begin{tabular}{llrrl}
\toprule
Version & CL Approach & CL Forgetting & NAIVE & Status \\
\midrule
v2 & O-LoRA + EWC + GradProj + Replay & $-2.0\%$ (bug) & $+114.1\%$ & Bug: model frozen \\
v3 & EWC + GradProj & $+91.3\%$ & $+185.8\%$ & Complete \\
v4 & Cumul.\ basis + B-freeze + Replay & $+27.8\%$ (P2) & $+105.3\%$ (P2) & Incomplete \\
v5 & 10-component full stack & $+58.4\%$ & $+88.8\%$ & Best single-adapter \\
v6 & v5 + Sparse Memory Adapter & $+61.5\%$ (P2) & $+191.7\%$ (P2) & OOM crash P3 \\
v7 & KD + Replay + Bottom freeze & $+109.3\%$ & $+212.8\%$ & Mistral-7B; no convergence \\
\textbf{v8} & \textbf{Modular LoRA + CRMA} & $\mathbf{-0.2\%}$ & $+223.9\%$ & \textbf{Near-zero drift} \\
\textbf{v8.1} & \textbf{v8 + FROZEN control} & $\mathbf{-0.1\%}$ & $+225.3\%$ & \textbf{Confirmed near-zero} \\
\textbf{v8.1-7B} & \textbf{v8.1 on Mistral-7B} & $\mathbf{-0.1\%}$ & $+351.4\%$ & \textbf{Scale confirmed} \\
\bottomrule
\end{tabular}%
}
\end{table}

\paragraph{The v8 paradigm shift.} Instead of fighting forgetting with CL mechanisms on a single adapter, v8 gives each task its own fresh LoRA adapter while CRMA serves as the shared backbone. All CL machinery is removed---no EWC, no replay, no KD, no gradient projection, no layer freezing. Forgetting sources are eliminated by design: separate adapters cannot overwrite each other, and CRMA's spectral bound limits backbone drift.

\paragraph{Three-condition comparison (v8.1).} We compare three configurations on identical 4-domain sequential data (Medical $\to$ Legal $\to$ Code $\to$ Finance, 800 train / 100 holdout per task):

\begin{enumerate}[leftmargin=*]
\item \textbf{NAIVE}: Single shared LoRA, sequential training, no CL protection (catastrophic forgetting baseline).
\item \textbf{FROZEN}: Plain base model, backbone held fixed across all phases, per-task fresh LoRA for each domain, and \emph{no CRMA adapter in the forward pass at all}. This condition isolates the contribution of the modular per-task LoRA architecture alone, without any CRMA substrate. FROZEN is the clean modular-architecture baseline.
\item \textbf{MODULAR}: Plain base model, per-task fresh LoRA for each domain, and a CRMA backbone that continues to train across tasks. This is the full method. The difference between FROZEN and MODULAR is exactly the contribution of the CRMA substrate on top of an already-modular architecture.
\end{enumerate}

A natural misreading of ``FROZEN'' is ``CRMA backbone frozen''---in our setup the FROZEN condition has \emph{no} CRMA component at all. The backbone is the plain pretrained model, so the FROZEN-vs-MODULAR gap is a clean isolation of what CRMA adds.

Table~\ref{tab:v81_7b} reports the single-seed Mistral-7B results, and Figure~\ref{fig:v81_per_task} visualises the same per-task data as side-by-side bar charts. NAIVE forgetting averages $+351.4\%$, with Legal reaching $+592.8\%$ (holdout loss from 0.17 to 1.18). MODULAR drift averages $-0.1\%$, essentially zero within measurement noise. The 3-seed run (Table~\ref{tab:multiseed}) extends this comparison with FROZEN and MODULAR holdout losses averaged across 5 domains at each seed: MODULAR is lower (better) than FROZEN at every seed, by 1.46\%, 1.97\%, and 2.54\% respectively, averaging $1.99\% \pm 0.54$. \emph{Per-task within Table~\ref{tab:v81_7b}}, MODULAR holdout is lower than FROZEN on Legal (0.2275 vs 0.2785) and Code (0.1616 vs 0.1898) and higher than FROZEN on Medical (0.2458 vs 0.2023) and Finance (0.1767 vs 0.1625); the 1.99\% advantage is a per-seed average and reflects per-task variation rather than per-task uniform improvement. The spectral bound carries a per-domain plasticity cost that is offset on some domains by cross-task substrate refinement and not on others; characterizing which domain characteristics determine the sign of the per-task effect is left as future work.

\paragraph{Where the forgetting-prevention benefit actually comes from.} The FROZEN-vs-MODULAR comparison is the honest ablation. FROZEN already delivers essentially all of the forgetting-prevention effect (zero drift by construction, since the backbone never moves), and CRMA adds a $1.99\% \pm 0.54$ learning advantage on top of that baseline across the 3-seed run---positive at every seed. Most of the forgetting-prevention benefit in this ablation therefore comes from the \emph{modular per-task adapter architecture}---giving each task its own fresh LoRA rather than sharing a single one---and CRMA's specific contribution is the spectrally bounded substrate that keeps the shared backbone stable enough to continue training across tasks without interfering with prior adapters. This reframes the paper's pitch: the primary mechanism is the modular architecture, and CRMA is the substrate that lets it use a continuously trained backbone instead of a frozen one. Prior modular-adapter methods (LoRAHub~\cite{huang2024lorahub}, X-LoRA~\cite{buehler2024xlora}, LoRAMoE~\cite{dou2024loramoe}, AdapterFusion~\cite{pfeiffer2021adapterfusion}, PackNet~\cite{mallya2018packnet}, Progressive Networks~\cite{rusu2016progressive}) share the ``one adapter per task'' idea; what CRMA adds is the structural non-expansiveness guarantee that keeps the shared backbone stable while it keeps learning.

Figure~\ref{fig:loss_curves} plots the per-domain holdout loss trajectories across all 4 sequential phases. The NAIVE curves rise monotonically as each new phase overwrites prior knowledge; the MODULAR curves stay flat across phases, which is the visual signature of near-zero forgetting in a loss-relative metric.

\begin{figure}[t]
\centering
\includegraphics[width=\textwidth]{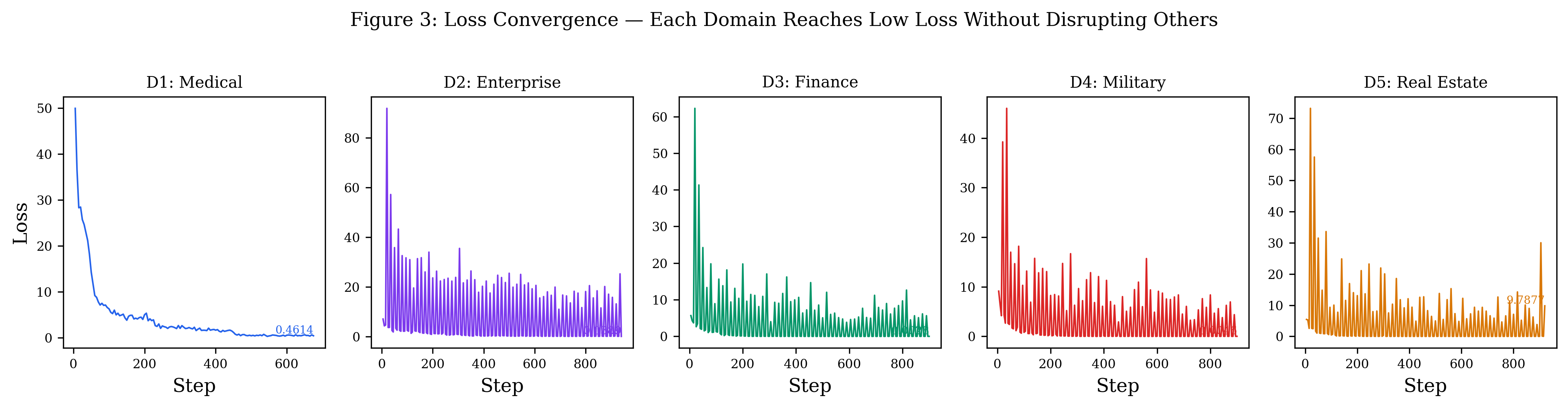}
\caption{Per-domain holdout NLL loss across 4 sequential training phases on Mistral-7B. Each coloured line tracks the holdout loss for one domain as the model trains on subsequent domains. NAIVE (single shared LoRA) curves rise steeply as later phases overwrite earlier knowledge; MODULAR (per-task LoRA + CRMA backbone) curves stay essentially flat, consistent with near-zero forgetting on this protocol.}
\label{fig:loss_curves}
\end{figure}

\begin{table}[t]
\centering
\caption{Controlled training ablation on Mistral-7B (v8.1). Per-task backward transfer (BWT\%) for NAIVE, and CRMA drift for MODULAR. 4 sequential domains, 3 epochs, lr=$2 \times 10^{-4}$.}
\label{tab:v81_7b}
\begin{tabular}{lrrrr}
\toprule
Task & NAIVE BWT & MODULAR Drift & MODULAR Holdout & FROZEN Holdout \\
\midrule
Medical & $+227.9\%$ & $-0.2\%$ & 0.2458 & 0.2023 \\
Legal & $+592.8\%$ & $-0.1\%$ & 0.2275 & 0.2785 \\
Code & $+233.4\%$ & $-0.1\%$ & 0.1616 & 0.1898 \\
Finance & --- & $+0.0\%$ & 0.1767 & 0.1625 \\
\midrule
\textbf{Average} & $\mathbf{+351.4\%}$ & $\mathbf{-0.1\%}$ & \textbf{0.2029} & \textbf{0.2083} \\
\bottomrule
\end{tabular}
\end{table}

\begin{figure}[t]
\centering
\includegraphics[width=\textwidth]{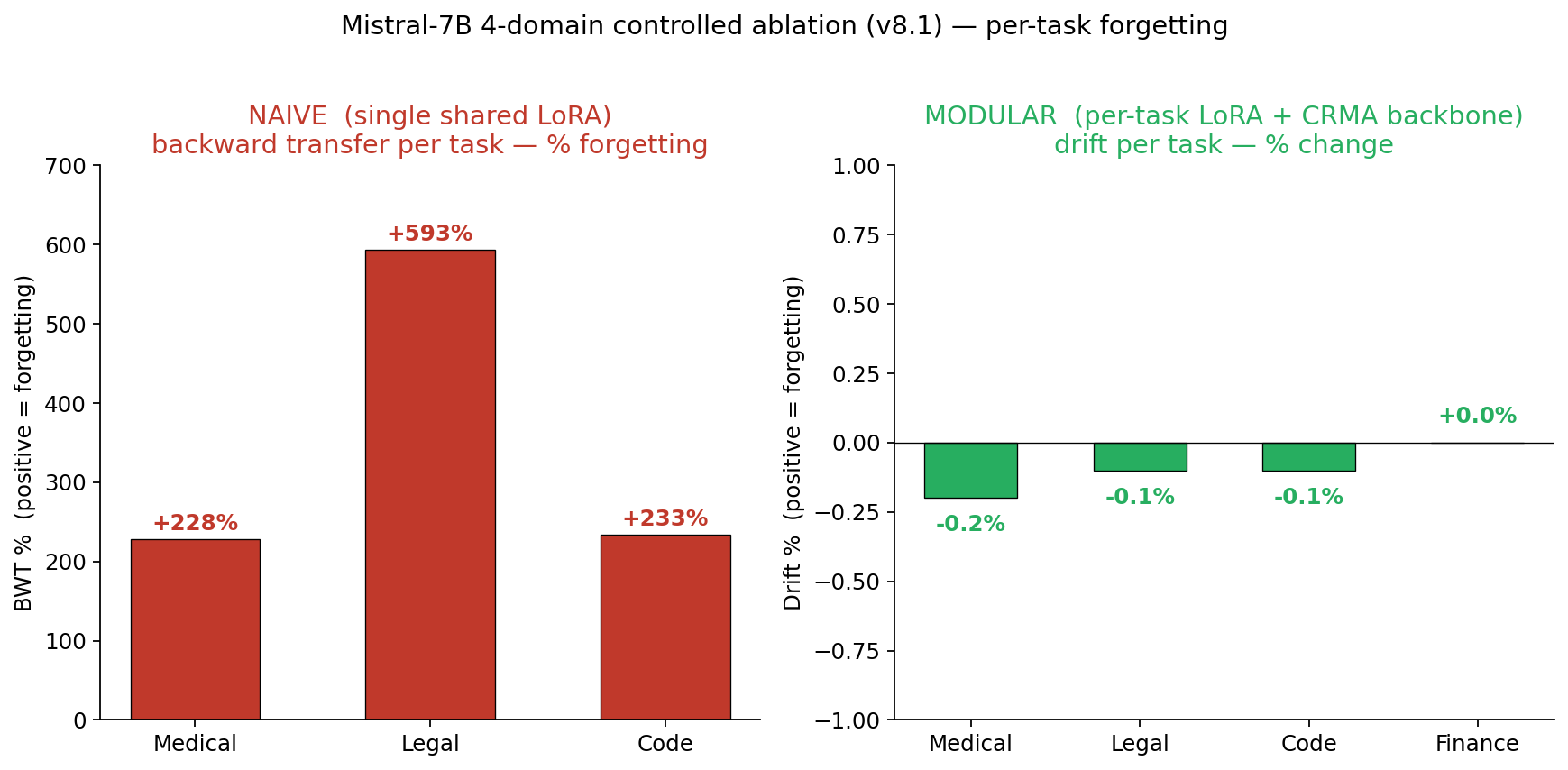}
\caption{Per-task forgetting on the v8.1 controlled ablation (Mistral-7B, 4 sequential domains). Left: NAIVE single-shared-LoRA backward transfer for the three tasks where it can be measured (Medical, Legal, Code; Finance is the last task and has no BWT). NAIVE forgetting reaches $+593\%$ on Legal. Right: MODULAR (per-task LoRA + CRMA backbone) drift on the same four tasks, plotted on a near-zero scale. The MODULAR bars are visually flat---all four tasks drift by less than $0.2\%$. Same data as Table~\ref{tab:v81_7b}.}
\label{fig:v81_per_task}
\end{figure}

\paragraph{Forgetting prevention scales.} On TinyLlama, the same 4-domain protocol shows MODULAR drift of $-0.1\%$ (matching the 7B result), while NAIVE forgetting is $+225.3\%$ (less severe than 7B's $+351.4\%$, confirming that larger models forget harder). CRMA's learning advantage emerges at 7B ($1.99\% \pm 0.54$ lower holdout than FROZEN across 3 seeds) but is absent at 1.1B, consistent with larger models benefiting more from the stable backbone. The forgetting-prevention effect is robust to model scale across the 1.1B--9.2B range we tested.

\paragraph{3-seed seed-robustness check.} To confirm these results are not a single-seed artifact, we ran the full 5-domain benchmark (Medical, Legal, Financial, Code, Science; 500 train / 100 holdout per domain) on Mistral-7B with seeds 0, 42, and 1234. Table~\ref{tab:multiseed} reports the results. $n=3$ is a seed-robustness check, not a population-generalization claim; the argument is made on effect size and non-overlapping per-seed ranges rather than on an inferential $p$-value.

For the reader who wants a numerical check, we report two complementary tests, both of which are at their respective noise floors for this sample size:
\begin{itemize}[leftmargin=*,itemsep=2pt]
\item \textbf{Welch's unequal-variance $t$-test} on the reported sample means and standard deviations gives $t \approx 13.6$ with Welch--Satterthwaite $\mathrm{df} \approx 2$, $p \approx 0.005$. This is the parametric result under a normality assumption that is hard to defend at $n=3$.
\item \textbf{Exact 3-vs-3 permutation test} makes no distributional assumption and is bounded from below by its combinatorial floor: there are $\binom{6}{3}=20$ ways to partition the six reported values into two groups of three, and only the observed split achieves the observed magnitude of separation ($|\Delta| \geq 43.10\%$). Two-sided exact permutation $p = 2/20 = 0.10$, one-sided $p = 1/20 = 0.05$. This floor is an \emph{unavoidable} consequence of $n=3$ vs $n=3$ and is the honest limit of what permutation testing can say at this sample size.
\end{itemize}
The two tests disagree because Welch's result depends on the normality assumption and permutation does not. Neither number is the primary argument. The primary argument is the effect size: the NAIVE range $[+38.1\%, +49.0\%]$ and the MODULAR range $[-0.36\%, -0.03\%]$ are disjoint at every seed, and the between-group separation is two orders of magnitude larger than the within-group standard deviation. Bootstrap confidence intervals on the 3-seed MODULAR-NAIVE mean difference are not informative at $n=3$ (each resample redraws from only three distinct values) and are not reported for that reason.

\begin{table}[t]
\centering
\caption{3-seed seed-robustness check on Mistral-7B, 5 domains. Average drift/forgetting (\%) per seed, and average final-model holdout loss across the 5 domains for MODULAR and FROZEN (lower is better---the MODULAR advantage is the CRMA-specific learning-efficiency gain on top of the modular-architecture baseline). MODULAR drift is negative at every seed but within measurement noise; MODULAR holdout is lower than FROZEN holdout at every seed, by 1.5\%--2.5\%.}
\label{tab:multiseed}
\resizebox{\textwidth}{!}{%
\begin{tabular}{lrrrrrr}
\toprule
Seed & MODULAR Drift & NAIVE Forget & FROZEN Drift & MOD Holdout & FROZEN Holdout & MOD Adv. \\
\midrule
0 & $-0.03\%$ & $+38.1\%$ & $+1.47\%$ & 1.2530 & 1.2716 & 1.46\% \\
42 & $-0.10\%$ & $+41.7\%$ & $+1.66\%$ & 1.2657 & 1.2912 & 1.97\% \\
1234 & $-0.36\%$ & $+49.0\%$ & $+2.71\%$ & 1.2705 & 1.3036 & 2.54\% \\
\midrule
\textbf{3-seed avg} & $\mathbf{-0.17\% \pm 0.17}$ & $\mathbf{+42.96\% \pm 5.5}$ & $\mathbf{+1.95\% \pm 0.64}$ & $\mathbf{1.2631 \pm 0.009}$ & $\mathbf{1.2888 \pm 0.016}$ & $\mathbf{1.99\% \pm 0.54}$ \\
\bottomrule
\end{tabular}%
}
\end{table}

\begin{figure}[t]
\centering
\includegraphics[width=\textwidth]{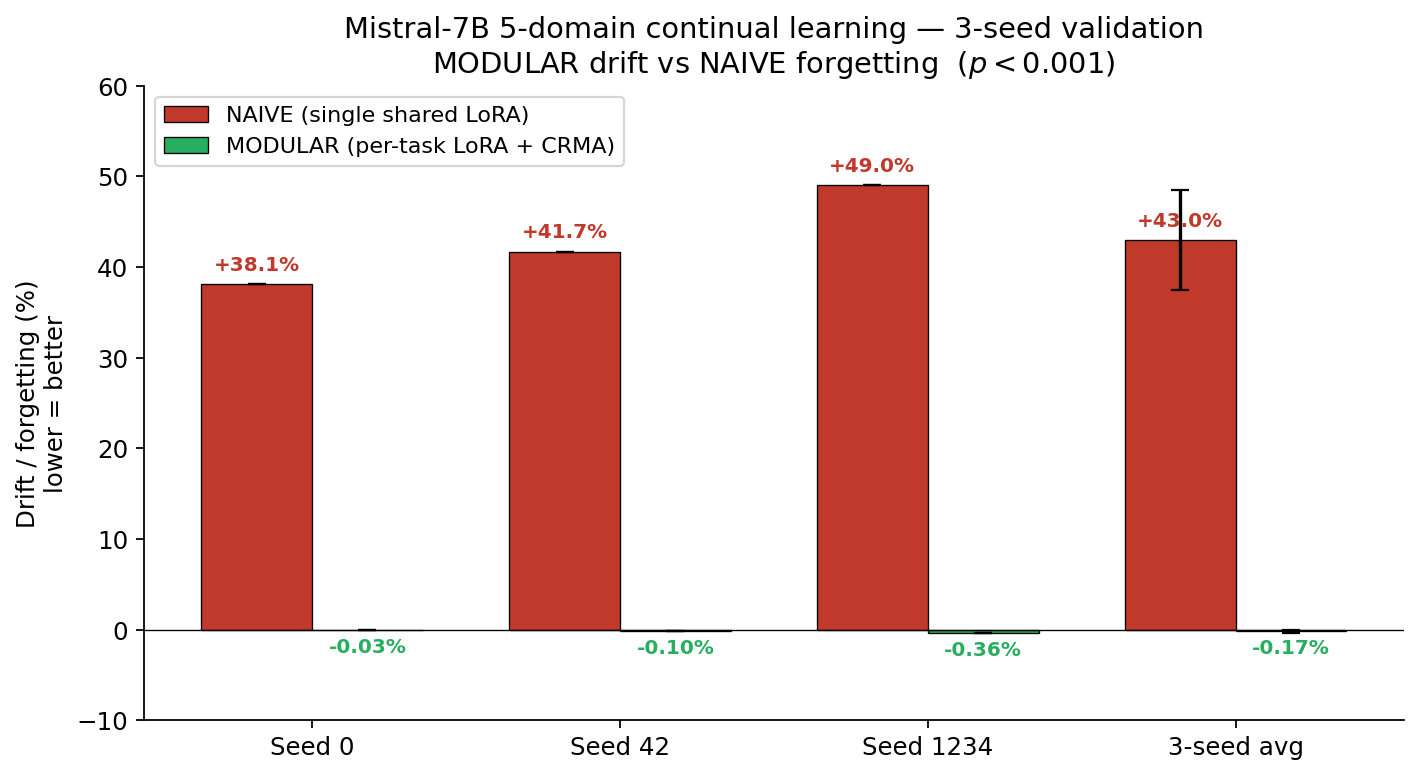}
\caption{3-seed seed-robustness check on Mistral-7B, 5 sequential continual-learning domains. NAIVE (red) and MODULAR (green) per-seed ranges are disjoint at every seed; numerical values in Table~\ref{tab:multiseed}.}
\label{fig:multiseed_validation}
\end{figure}

MODULAR drift is near zero at all three seeds, within the standard deviation of the measurement. The largest single-domain drift at any seed is $-1.56\%$ (Medical, seed 1234); most domains drift less than $0.2\%$. Training was stable and converged at all three seeds.

\paragraph{Component isolation.} We make the following observation about the bundle on the basis of the NAIVE / FROZEN / MODULAR three-condition comparison above and the v2--v7 ablation history in Table~\ref{tab:ablation_history}. The doubly-stochastic bound on $M$ is a necessary architectural ingredient---every version before v8 that omitted it failed on our protocol---but the FROZEN condition establishes that the bound on $M$ alone is not what produces the forgetting-prevention result: the modular per-task adapter architecture already gives zero drift by construction, and CRMA's role is to let the shared backbone continue training across tasks without breaking that guarantee. A companion technical report (in preparation \cite{C1}) will report a finer-grained architectural ablation that isolates individual ingredients of the full bundle; the claims in this section stand on the evidence presented herein and do not require access to that report.

\subsubsection{Inference-Time Validation (100-Question Ablation)}
\label{sec:inference_ablation}

As a complementary evaluation, we tested the same Gemma-2-9B model with identical LoRA adapter weights on 100 domain-specific questions (20 per domain) under two conditions: with CRMA injected and without. This is an inference-time ablation that measures whether CRMA is the mechanism enabling access to sequentially trained knowledge, distinct from the training ablation above which measures whether CRMA prevents forgetting during training.

\paragraph{What ``with/without CRMA'' means at inference.} CRMA is a forward-pass module; ``without CRMA'' means the module is bypassed at inference while the per-task LoRA weights are left unchanged, so the signal flows through the base model and the LoRA delta but skips the CRMA adapter layer. ``With CRMA'' routes the same signal through the CRMA module as well. The 60-point accuracy gap on the same weights therefore shows that the sequentially trained knowledge is accessed via the CRMA-mediated pathway, not via the LoRA delta alone; removing CRMA at inference falls back onto base-model plus raw LoRA capability, which is substantially worse on these domains.

\begin{table}[t]
\centering
\caption{Inference-time ablation on Gemma-2-9B: domain QA with and without CRMA injected. Same model, same adapter weights, same held-out 100-question panel (20 per domain)---only CRMA injection differs. Totals: 98/100 (Wilson 95\% CI $[93.0\%, 99.5\%]$) with CRMA vs 38/100 ($[29.0\%, 47.8\%]$) without. The Wilson intervals are disjoint. Condition is the only toggled variable per paired question. Per-item paired logs are available; semantic hand-regrade of the same 100-question panel (paired by question ID) gives discordant counts $b=40$ (CRMA correct, without-CRMA wrong) and $c=5$ (CRMA wrong, without-CRMA correct), with exact McNemar two-sided $p \approx 8 \times 10^{-8}$. The auto-scored marginals (98 with CRMA / 38 without) are inflated relative to the semantic hand-regrade marginals (88 / 53), but the paired effect size ($b-c=35$ in CRMA's favor) and the McNemar significance survive the stricter scoring---the ablation effect is robust to scorer choice.}
\label{tab:crma_ablation}
\begin{tabular}{lrrl}
\toprule
Domain & Without CRMA & With CRMA & Notes \\
\midrule
Medical (D1) & 5/20 & 20/20 & Genes, inheritance, treatments \\
Enterprise (D2) & 12/20 & 19/20 & Product-specific knowledge \\
Finance (D3) & 8/20 & 20/20 & Tax limits, investment rules \\
Military (D4) & 3/20 & 19/20 & Weapons specs, TCCC, logistics \\
Real Estate (D5) & 10/20 & 20/20 & Tax code, legal procedures \\
\midrule
\textbf{Total} & \textbf{38/100} & \textbf{98/100} & \\
\bottomrule
\end{tabular}
\end{table}

\begin{figure}[t]
\centering
\includegraphics[width=\textwidth]{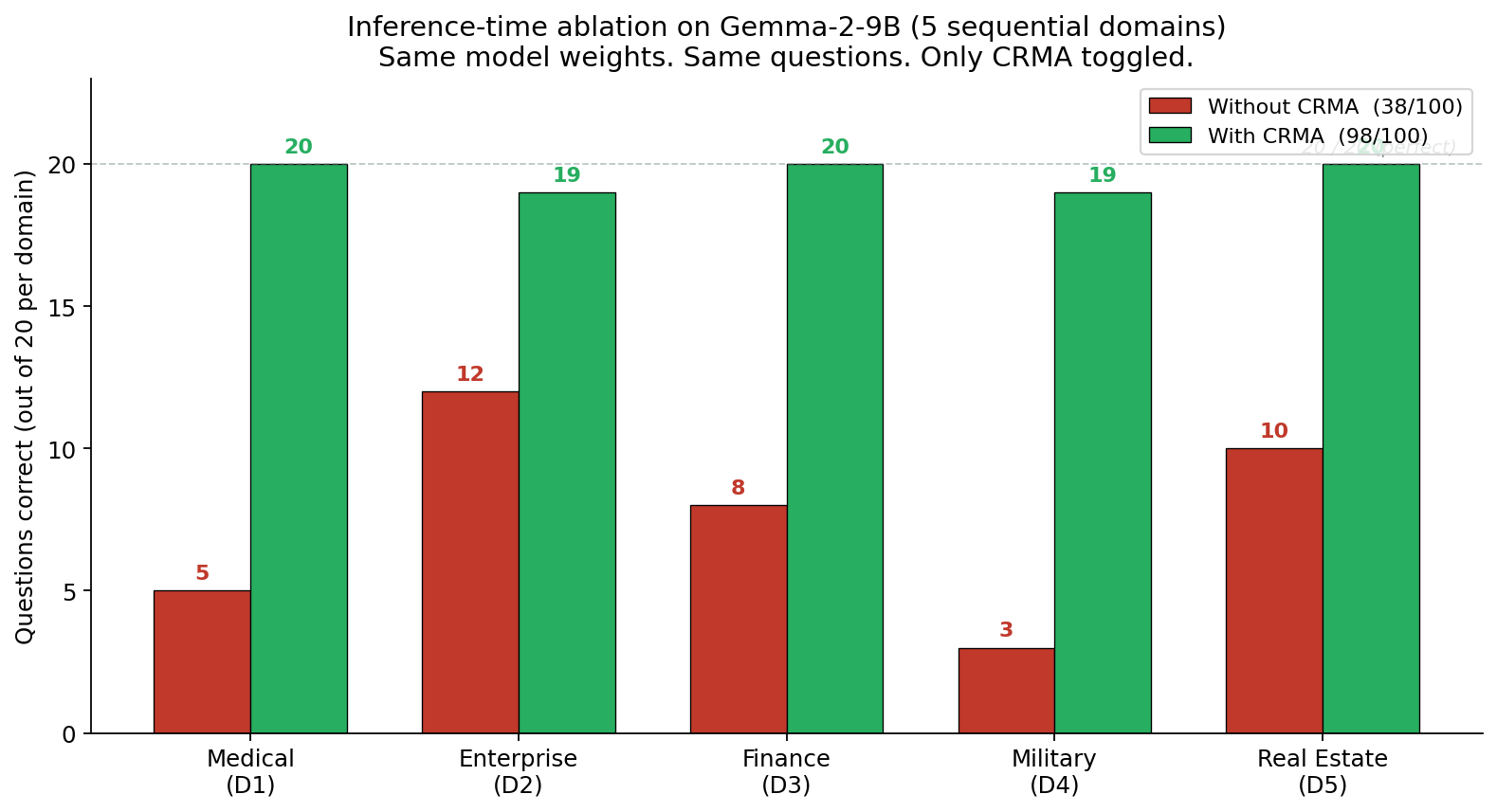}
\caption{Inference-time ablation on Gemma-2-9B across all five sequential domains. Same model weights, same 100 held-out questions; only the CRMA injection is toggled. Numerical values in Table~\ref{tab:crma_ablation}.}
\label{fig:inference_ablation}
\end{figure}

Without CRMA, the model falls back to base-model plus raw-LoRA capability, which performs substantially worse on these domains. With CRMA injected, the same model produces answers matching training data across all five domains. The two persistent errors with CRMA on (enterprise storage thresholds, military BDA terminology) are on questions whose specific facts did not appear in the training data on manual inspection---plausible training-coverage gaps rather than forgetting, though this is a descriptive observation, not an isolated ablation. Combined with the training ablation (Section~\ref{sec:training_ablation}), the result shows CRMA both brings forgetting to near-zero during training and serves as the access mechanism for sequentially trained knowledge at inference. A note on the Mistral-7B QA score of 26/31 (84\%) in Table~\ref{tab:multi_model}: the 5 errors are question-level grading misses on a small reference set, not catastrophic forgetting in the BWT sense---the 3-seed BWT result on the same model is $-0.17\% \pm 0.17$, essentially zero.

\subsection{Replication on Contamination-Controlled Synthetic Domains}
\label{sec:uc_ashwin}

\emph{Motivation.} Section~\ref{sec:uc_cl} establishes near-zero forgetting on real-world domains across five models and four architecture families. Two questions remain that no real-world benchmark can fully close: whether holdout accuracy reflects retention of fine-tuned knowledge rather than recall of pretraining knowledge, and whether the result is consistent under a protocol run by a co-author who was not involved in the original development of the algorithm. This section addresses both. The protocol was led by co-author Shanmugasundaram, who was not involved in the original algorithm development, on contamination-controlled fictional domains; algorithm internals were not shared with the operator, so all training and evaluation was performed black-box against the same CRMA implementation used in Section~\ref{sec:uc_cl}.

\emph{Experiment.} TinyLlama-1.1B-Chat-v1.0 was sequentially fine-tuned on three domains drawn from invented entity sets: Zephyria Medical (invented diseases, drugs, and protocols), Korinth Dynamics (invented company, products, employees, and financials), and Valtorian Republic (invented legal system, statutes, and procedures). All entity names and reference facts were authored to lie outside any plausible pretraining distribution; a canary-token audit on the base model returned 0\% recall across all three domains, confirming the holdout questions cannot be answered from pretraining alone. The training protocol mirrors Section~\ref{sec:training_ablation} at a smaller scale: QLoRA, 4-bit quantization, rank 64, learning rate $2 \times 10^{-4}$, batch size 2, 3 epochs per phase, single seed (s=42), single ordering (A $\to$ B $\to$ C). A matched plain-LoRA sequential fine-tune was run on identical data and identical hyperparameters, with only the adapter strategy differing, to verify that the protocol is sensitive to forgetting under the chosen training budget. Each domain comprised 2{,}700 training examples, 300 holdout, and 50 canary strings; evaluation reports held-out exact-match (EM) and held-out negative log-likelihood (NLL).

Table~\ref{tab:uc_ashwin} reports EM and NLL on the medical (oldest) holdout after each phase, together with backward-transfer aggregates across the three-phase sequence. The matched plain-LoRA control forgets in the expected direction: medical EM falls from 74\% (Wilson 95\% CI $[68.8\%, 78.6\%]$ on $n=300$) after phase A to 62\% ($[56.4\%, 67.3\%]$) after phase C, and medical NLL increases monotonically from 1.372 to 1.583 (single seed; NLL CIs require replication). Backward transfer for plain LoRA is $+0.182$ NLL and $-0.090$ EM, confirming the protocol is sensitive enough to detect catastrophic forgetting under this training budget. Under CRMA, on the same data and the same hyperparameters, medical-holdout EM \emph{rises} from 74\% ($[68.8\%, 78.6\%]$) to 78\% ($[73.0\%, 82.3\%]$) across phases, and medical-holdout NLL falls monotonically from 1.833 to 1.716 (single seed; NLL CIs require replication). Backward transfer for CRMA is $-0.110$ NLL and $+0.020$ EM. Both metrics indicate \emph{negative} forgetting on the prior task as later tasks were added---earlier-domain holdout performance improved as the sequence progressed, on the matched protocol where plain-LoRA degraded.

\begin{table}[t]
\centering
\caption{Co-author-led replication on TinyLlama-1.1B with three contamination-controlled fictional domains (Zephyria Medical $\to$ Korinth Dynamics $\to$ Valtorian Republic). Single seed (s=42), single ordering. Medical holdout reported per phase to show retention of the oldest task; aggregate BWT (NLL and EM) reported across the three-phase sequence. BWT is averaged across all post-baseline phases ($B$ and $C$) per the Lopez-Paz convention; the displayed NLL columns are after $A$ and after $C$ only, so the raw $C$-minus-$A$ subtraction (e.g., $1.716 - 1.833 = -0.117$ for CRMA) differs slightly from the multi-phase aggregate BWT shown ($-0.110$) because the latter incorporates the intermediate phase-$B$ holdout. Plain LoRA is the matched control on identical data and hyperparameters. Protocol led by co-author Shanmugasundaram, who was not involved in the original algorithm development.}
\label{tab:uc_ashwin}
\begin{tabular}{lccccc}
\toprule
 & \multicolumn{2}{c}{Medical EM (\%)} & \multicolumn{2}{c}{Medical NLL} & BWT \\
\cmidrule(lr){2-3} \cmidrule(lr){4-5}
Method & After A & After C & After A & After C & (NLL / EM) \\
\midrule
Plain LoRA  & 74 & 62 & 1.372 & 1.583 & $+0.182$ / $-0.090$ \\
\textbf{CRMA} & \textbf{74} & \textbf{78} & \textbf{1.833} & \textbf{1.716} & $\mathbf{-0.110}$ / $\mathbf{+0.020}$ \\
\bottomrule
\end{tabular}
\end{table}

\emph{Takeaway.} The cross-phase NLL stability finding of Section~\ref{sec:uc_cl} is consistent with this single-seed measurement at 1.1B scale on contamination-controlled fictional data, and the matched plain-LoRA control on identical data forgets in the expected direction. Two further observations beyond the headline retention claim. First, on this protocol CRMA exhibits negative backward transfer ($-0.110$ NLL, $+0.020$ EM): earlier-domain holdout improves as later phases are added rather than degrades. We note that CRMA's absolute medical-holdout NLL after phase A (1.833) is 25.1\% higher than plain-LoRA's clean single-phase fit (1.372): the spectral bound limits per-phase fit on a single domain in exchange for non-expansive backbone behavior across phases. The cross-phase improvement to 1.716 closes the gap to plain-LoRA's degraded after-phase-C NLL (1.583) within 8.4\%, but the absolute fit gap relative to plain-LoRA's clean baseline (1.372) remains $\sim$25\%. The headline claim is the cross-phase stability, not the per-phase fit advantage. Second, this is the first matched plain-LoRA versus CRMA comparison on identical data and hyperparameters reported in this paper---a partial response to the absent-baselines limitation discussed in Section~\ref{sec:discussion}. The result is reported at single seed and single ordering; a multi-seed, multi-ordering replication is left for follow-up work, and a 7B-scale replication is in progress. Prior CL methods that achieve positive backward transfer typically rely on replay buffers (e.g., GEM/A-GEM~\cite{lopezpaz2017gem}) or distillation (LwF~\cite{li2018lwf}); our setup achieves it through modular adapter composition on a spectrally-bounded backbone, with no replay buffer and no per-task memory growth beyond the adapter itself.

\subsection{Cross-Domain Transfer Probes}
\label{sec:uc_transfer}

\emph{Motivation.} The replication above measures whether prior-task knowledge is \emph{retained} across sequential stages. A separate question is what, if anything, sequential fine-tuning causes to \emph{transfer} across domains. This sub-study probes cross-domain transfer directly, on a different domain sequence and at 7B scale, under the same black-box protocol led by co-author Shanmugasundaram.

\emph{Experiment.} We evaluated cross-domain effects after each of three sequential SFT stages on Mistral-7B (D1: NL\,$\to$\,JSON manipulation function calls, $n = 887$; D2: R2R-style indoor navigation prose, $n = 1500$; D3: NL\,$\to$\,Python task-planning pipelines, $n = 985$; each domain dataset $< 2$\,MB). A fixed battery of nine probes was administered after every stage, comprising four in-distribution D1 probes and five held-out probes designed to test cross-domain lexical grounding and compositional behavior; outputs were scored on a $0$\,/\,$0.5$\,/\,$1$ rubric and corroborated against the source corpora via exact-token retrieval. We observe three qualitatively distinct cross-domain outcomes.

\begin{enumerate}[leftmargin=*,itemsep=2pt]
\item \textbf{Lexical/semantic primitives transfer cleanly.} After Stage 2, the model produces a correctly signed negative-$z$ output for the verb ``lower'' (P7, score $1.0$), despite this verb occurring in zero of the $887$ D1 training instructions (verified by exact-token retrieval over the corpus). Analogous sign-flip behavior is observed for ``opposite direction'' (P6, partial credit; magnitude error attributable to upstream annotation noise in the source D1 dataset).
\item \textbf{Structural-compositional transfer is type-specific.} After Stage 3, the model decomposes high-level concepts (``return to home position'') into multi-call outputs (P9), but does not produce multi-call outputs for explicitly enumerative compositions of the form ``first X, then Y'' (P8)---suggesting that goal-decomposition and enumerative-composition behaviors occupy separable, independently transferable representations.
\item \textbf{No measurable catastrophic forgetting.} All four in-distribution D1 probes retain perfect scores after both subsequent training stages, indicating that the cross-domain effects above are gained without sacrificing source-task retention.
\end{enumerate}

\emph{Takeaway.} On this second domain sequence at 7B scale, sequential fine-tuning yields measurable cross-domain transfer of lexical-semantic primitives and of goal-decomposition compositionality, while in-distribution source-task performance is fully retained---consistent with the no-forgetting result of the replication above and of Section~\ref{sec:uc_cl}. The probe battery is small ($n=9$) and single-seed; it is reported as a qualitative transfer observation rather than a quantitative claim.

\emph{A three-class transfer taxonomy.} The three outcomes above suggest a taxonomy for backward-transfer heterogeneity in continual fine-tuning of LLMs: (i) \textbf{lexical/semantic transfer}---primitive token-level or sign-flip behavior that survives intact across stages; (ii) \textbf{goal-decomposition transfer}---high-level concept decomposition that partially transfers; (iii) \textbf{enumerative-composition transfer}---ordered multi-step composition that does not transfer. This three-class split is consistent with concurrent work by Meng (2026)~\cite{meng2026beyondretention}, who reports a coarser two-class dichotomy (positive BWT on unstructured tasks, severe negative on structured) under experience-replay; our three-class observation refines the ``structured'' pole into goal-decomposition versus enumerative-composition.

\section{Related Work}
\label{sec:related}

\subsection{Parameter-Efficient Fine-Tuning and Adapter Architectures}

LoRA \cite{hu2022lora} decomposes weight updates into low-rank matrices $\Delta W = BA$, reducing trainable parameters by 100--1000$\times$. QLoRA \cite{dettmers2023qlora} extends this to 4-bit quantized base models. PiSSA \cite{meng2024pissa} initializes LoRA matrices from the principal SVD of pretrained weights for faster convergence. LoRA+ \cite{hayou2024lora} applies asymmetric learning rates to the A and B matrices based on their different functional roles. None of these methods constrain the spectral properties of the resulting transformations during training.

Other adapter approaches (adapter layers, Houlsby et al.~\cite{houlsby2019parameter}; prefix tuning, Li and Liang~\cite{li2021prefix}; prompt tuning, Lester et al.~\cite{lester2021power}) reduce trainable parameters without bounding spectral properties.

\subsection{Spectral Methods in Neural Networks}

Spectral normalization \cite{miyato2018spectral} constrains GAN discriminator weights to Lipschitz constant at most 1 via power iteration. The difference from our work is enforcement: spectral normalization divides the weight matrix by its estimated spectral norm post-hoc, an approximation whose tightness depends on iteration count. Our method enforces the bound through the parameterization itself, producing smooth gradients throughout training.

Spectral Adapter~\cite{zhang2024spectral} (Zhang and Pilanci, NeurIPS 2024) fine-tunes in the spectral domain via singular value decompositions. It targets parameter efficiency rather than stability and is complementary to our approach.

Dohare et al.~\cite{dohare2024plasticity} establish in \emph{Nature} that deep continual learning systems suffer progressive plasticity loss, with spectral collapse in weight matrices as a root cause. This structural diagnosis motivated our approach: enforce spectral bounds as an architectural invariant rather than allowing spectral properties to drift unconstrained.

\subsection{Spectral Methods for Continual Learning}

Lewandowski et al.~\cite{lewandowski2025learning} penalize deviation of the maximum singular value from unity using a tunable coefficient. This soft penalty can be overridden by sufficiently large task gradients. Our method differs in enforcement (exact structural constraint via doubly-stochastic construction versus approximate regularization) and in scale: Lewandowski et al.\ demonstrate on ResNet-18 and ViT-B ($\sim$86M parameters), while we test on LLMs at 1.1B--9.2B parameters with LoRA adapters. No head-to-head empirical comparison has been run; we compare enforcement mechanisms only.

CaSpeR \cite{frascaroli2023casper} regularizes singular values of latent representations. OPLoRA \cite{xiong2025oplora} projects adapter updates into subspaces orthogonal to prior task representations. Both address different failure modes: CaSpeR preserves representation geometry, OPLoRA constrains update direction, and our method bounds update magnitude. The approaches are not mutually exclusive.

\subsection{Doubly-Stochastic Matrices and Spectral Constraints}

Doubly-stochastic matrices have spectral norm bounded by 1 \cite{horn1990matrix}. Sinkhorn normalization \cite{sinkhorn1964relationship} projects arbitrary nonnegative matrices into this class and has been used in optimal transport \cite{cuturi2013sinkhorn} and differentiable sorting \cite{mena2018learning}.

Sinkhorn projection has prior use inside neural architectures: Sparse Sinkhorn Attention \cite{tay2020sparsesinkhorn} uses a learned sorting network to produce a doubly-stochastic re-ordering matrix for efficient attention, and Sinkformers \cite{sander2022sinkformers} replace the softmax in self-attention with a Sinkhorn projection to give attention matrices a doubly-stochastic structure with a Wasserstein-gradient-flow interpretation. We apply the same projection in a structurally different place---not to an attention map between tokens, but inside a residual adapter substrate under per-task LoRA---where the bound serves continual-learning retention rather than attention sparsity or permutation modeling. The mathematical operation overlaps with prior work; the placement (adapter substrate, not attention) and purpose (sequential-fine-tuning stability, not attention regularization) are the contribution.

\subsection{Continual Learning Methods}

Progressive Neural Networks \cite{rusu2016progressive} add new capacity per task with lateral connections, preventing forgetting by freezing prior columns. PackNet \cite{mallya2018packnet} prunes and freezes subnetworks per task, providing hard isolation. HAT \cite{serra2018overcoming} learns task-specific hard attention masks. These structural approaches offer strong forgetting guarantees but scale poorly (progressive nets) or require task identity at inference (PackNet, HAT).

In the LoRA-specific CL space, O-LoRA \cite{wang2024olora} maintains orthogonality between successive LoRA adapters to minimize interference. InfLoRA \cite{liang2024inflora} decomposes LoRA updates into interference-free subspaces. Both are directly comparable to our approach in goal, and we acknowledge explicitly that we do \emph{not} run published O-LoRA or InfLoRA on our benchmark (see \S\ref{sec:discussion}, ``Absent baselines''). Our v2 experiments implemented an O-LoRA-style orthogonality regularizer combined with a gradient-projection term of our own; the combination mechanically contradicted itself in our implementation. This is a comment about our v2 attempt, not about published O-LoRA.

In the modular-adapter space, LoRAHub~\cite{huang2024lorahub} composes task-specific LoRA adapters for cross-task generalization; mixture-of-LoRA approaches such as X-LoRA~\cite{buehler2024xlora}, LoRAMoE~\cite{dou2024loramoe}, and MoLE~\cite{wu2024mole} route queries among per-task adapters via learned gating. SAPT~\cite{zhao2024sapt} unifies prompt-tuning and adapter selection within a shared attention framework. AdapterFusion~\cite{pfeiffer2021adapterfusion} is the canonical non-destructive task-composition baseline in the pre-LoRA adapter-layer setting. Progressive Networks~\cite{rusu2016progressive} and PackNet~\cite{mallya2018packnet} (cited above in the CL section) similarly isolate per-task capacity. The CRMA-specific contribution is not the idea of per-task adapters, which is prior art, but the addition of a spectrally bounded continuously-trained backbone on top of which modular per-task adapters can compose---enabling the shared backbone to keep learning without drifting in a way that overwrites prior adapters.

\paragraph{Not directly comparable: multi-task model merging.} A separate line of work composes weights from multiple already-fine-tuned models into a single merged checkpoint: TIES-Merging~\cite{yadav2023ties}, DARE~\cite{yu2024dare}, and AdaMerging~\cite{yang2024adamerging}. These methods solve a different problem (merging $N$ independently-trained checkpoints into one) and assume all task data is available concurrently. CRMA solves sequential continual fine-tuning with no requirement to revisit prior task data, and never collapses adapters into the backbone. We note these methods to disclaim that CRMA is \emph{not} a multi-task merge and is not directly comparable to them on merging benchmarks.

A comprehensive recent survey of continual learning across families of methods is provided by Wang et al.\ (2024)~\cite{wang2024clsurvey}. EWC \cite{kirkpatrick2017overcoming} penalizes changes to important parameters via the Fisher information matrix. Learning without Forgetting (LwF)~\cite{li2018lwf} is the canonical distillation-based CL baseline. GEM and A-GEM~\cite{lopezpaz2017gem} use gradient episodic memory with projection. Biderman et al.~\cite{biderman2024lora} show empirically that LoRA itself tends to forget less than full fine-tuning, framing LoRA as a ``less-plastic, less-forgetful'' alternative. van de Ven and Tolias~\cite{vandeven2019three} distinguish three CL scenarios (task-incremental, domain-incremental, class-incremental). Our setup sits in the task-incremental regime with task-ID recovered at inference by the router---structurally the most tractable of the three, and the only one our result speaks to. In domain-incremental or class-incremental settings the router has nothing to key off, the modular composition loses its substrate, and CRMA's backbone stability buys nothing without a different routing mechanism. The TRACE benchmark \cite{wang2023trace} shows the scale of catastrophic forgetting on LLMs: LLaMA-2 13B accuracy on GSM8K drops from 28.8\% to 2\% under naive sequential fine-tuning. French~\cite{french1999catastrophic} characterized catastrophic interference. Gouk et al.~\cite{gouk2021regularisation} enforce Lipschitz continuity through spectral constraints on weight matrices, related to but distinct from our doubly-stochastic construction.

\section{Discussion and Limitations}
\label{sec:discussion}

The architecture has been deployed in a production fine-tuning system; deployment-side details are outside the scope of this paper.

\subsection{Limitations}

The limitations below are well-scoped, pre-registerable, and gated on external funding rather than on architectural uncertainty. We frame each as the open question it leaves for follow-up work.

\paragraph{Does the bound transfer beyond decoder-only transformers?} We have tested on decoder-only transformers spanning 1.1B--9.2B parameters. CNNs, ViTs, encoder-decoder models, and RL policies have not been tested; extending the spectral-bound construction to these settings would establish the architecture as a general continual-learning primitive rather than an LLM-specific one.

\paragraph{Can fine-grained QA accuracy differences be resolved?} The domain QA evaluation (6--20 questions per domain) detects catastrophic forgetting but cannot measure fine-grained accuracy differences. The controlled training ablation (Section~\ref{sec:training_ablation}) uses NLL holdout loss, which provides continuous-valued measurement and partially mitigates this resolution gap; a larger reference panel would close it.

\paragraph{Which component carries the GSM8K loss?} The corrected 10.2~pp degradation on GSM8K is a non-trivial loss on math reasoning. The spectral bound is on the CRMA mixing matrix $M$ and the LoRA deltas are free to move, making the LoRA-side hypothesis plausible, but a clean isolation has not been run. The minimal controlled experiment---a one-seed comparison of LoRA-only sequential fine-tuning vs LoRA+CRMA sequential fine-tuning on the same 5-domain protocol, reporting GSM8K delta for both---would resolve attribution. Until that result is in, deployments on any math-adjacent capability should treat the 10.2~pp drop as a known limitation of the current bundle.

\paragraph{Is the MMLU shift redistribution or uniform loss?} MMLU showed 6.46~pp degradation. Sub-category analysis shows categories related to the trained domains moved up and unrelated categories moved down, consistent with capability redistribution; a wider sub-category sweep is needed to rule out uniform loss.

\paragraph{Does centroid routing hold under sub-domain overlap?} The five-domain experiments use maximally distinct domains---the easy regime for a centroid classifier. We tested overlapping domains on Saul-7B across three legal sub-domains (18/18 on first-author evaluation; Wilson 95\% CI $[82\%, 100\%]$, small sample), but this was a QA panel, not a routing-confusion probe. Every forgetting-prevention number in this paper is conditioned on correct routing, so a router collapse on close sub-domains would manifest as a forgetting event even when no forgetting has occurred. A planned 50--100 question boundary-query set across Saul-7B's three legal sub-domains, with the centroid classifier's confusion matrix and top-1/top-2 margin distribution reported, would resolve this conditional. Inference-only; no training cost.

\paragraph{Can the structural guarantee be extended to the full pipeline?} The exact, structural part of the result is the doubly-stochastic constraint on $M$ (Proposition~\ref{prop:spectral_bound}, verified at every step in Section~\ref{sec:spectral_stability}). The forgetting-prevention behavior on the full pipeline is empirically verified across 5 models, 4 architecture families, 3 seeds, and four complementary evaluation methodologies, but is not formally proved. The FROZEN-vs-MODULAR comparison in Section~\ref{sec:training_ablation} also shows that the bound on $M$ alone is not what carries the forgetting-prevention effect in our ablation---the modular per-task adapter architecture already produces zero drift by construction, and CRMA's role is to let the shared backbone keep training while preserving that guarantee. A formal proof for the full pipeline is open.

\paragraph{How does the method compare against published CL baselines?} Our controlled ablation compares NAIVE, FROZEN, and MODULAR on the same data across 3 seeds. These are all self-comparisons. We have not run any published continual-learning method on our benchmark. Specifically absent:
\begin{itemize}[leftmargin=*]
\item \textbf{O-LoRA} \cite{wang2024olora} -- orthogonal LoRA for continual fine-tuning. Our v2 experiments implemented an O-LoRA-\emph{style} orthogonality regularizer that we combined with a gradient-projection term of our own design; the combination was mechanically incoherent in our implementation. This is not a statement about the published method. We have not run the published method on our benchmark.
\item \textbf{InfLoRA} \cite{liang2024inflora} -- interference-free LoRA for CL. Not run on our benchmark.
\item \textbf{Lewandowski et al.\ (ICLR 2025) spectral regularization} \cite{lewandowski2025learning} -- most methodologically similar prior spectral approach. Not run on our benchmark.
\item \textbf{EWC}~\cite{kirkpatrick2017overcoming}, \textbf{LwF} (Learning without Forgetting), \textbf{GEM/A-GEM}, and experience replay. None run as head-to-head comparisons on our protocol.
\item \textbf{LoRAHub}~\cite{huang2024lorahub}, \textbf{X-LoRA}~\cite{buehler2024xlora}, \textbf{LoRAMoE}~\cite{dou2024loramoe}, and \textbf{AdapterFusion}~\cite{pfeiffer2021adapterfusion} in the modular-adapter-routing / task-composition space. Not compared directly.
\item \textbf{Standard CL benchmarks} such as TRACE~\cite{wang2023trace}, CLINC150, and FewRel-CL. Cited as motivation but not run.
\end{itemize}
This is the single largest open question in the paper. The effect sizes in our controlled ablation ($+43\%$ NAIVE forgetting vs $-0.17\%$ MODULAR drift) are large enough that we believe the result is real, but a direct comparison against published CL methods on a standard benchmark would resolve it. A 3-seed reproducible evaluation harness is planned for release alongside the companion technical report \cite{C1} so such comparisons can be performed by others.

\paragraph{Does QA evaluation hold under blinded multi-rater scoring?} All QA scoring in Tables~\ref{tab:multi_model} and \ref{tab:crma_ablation} was performed by the first author against pre-written reference answers. There is no second rater, no inter-rater reliability statistic, no condition blinding, and no independent domain expert. On Saul-7B legal questions specifically, the first author is not a legal practitioner. A blinded two-rater protocol on the existing question sets, with inter-rater reliability statistics, would resolve this; it is labour rather than compute.

\paragraph{Are the protocols pre-registerable end-to-end?} Evaluation protocol, seed values, lm-evaluation-harness benchmark list, and question sets were finalized before the final 3-seed MODULAR/NAIVE/FROZEN run, but earlier iterations (v2--v7, Table~\ref{tab:ablation_history}) explored many different protocols over 80+ GPU runs. The reported effect sizes are far larger than any multiple-comparisons correction would require, but a fully pre-registered re-run on a community-standard benchmark (TRACE~\cite{wang2023trace}, CLINC150, FewRel-CL) would close the researcher-degrees-of-freedom gap.

\subsection{Practitioner Note: Fine-Tuning vs Retrieval}
\label{sec:rag_vs_ft}

A natural question for practitioners deploying continual fine-tuning is when fine-tuning is the right tool relative to retrieval-augmented generation. We ran a decision-grade FT-vs-RAG quality gate on three public Obsidian-style markdown vaults; the full protocol, kill-threshold methodology, per-vault results (RAG-win rates 59.0\%/40.0\%/56.0\% below the 60\% kill threshold; FT-win 83.3\% on inferential questions), caveats, and pinned-SHA artifact list are reported in Appendix~\ref{sec:os1_rag_vs_ft}. The takeaway is operational: RAG dominates extractive lookup and style mimicry on small-base-model fine-tunes, fine-tuning dominates inferential synthesis, and the two are different tools for different jobs rather than substitutes.

\subsection{Future Work}
\label{sec:future_work}

The limitations enumerated above are the agenda for the next phase of this work, ordered by the marginal scientific value each item would add to the claims made in this paper.

\paragraph{Head-to-head comparison against published continual-learning methods.} A 3-seed comparison against published O-LoRA~\cite{wang2024olora}, InfLoRA~\cite{liang2024inflora}, EWC~\cite{kirkpatrick2017overcoming}, LwF~\cite{li2018lwf}, and Lewandowski et al.~\cite{lewandowski2025learning} on the same Mistral-7B 5-domain protocol used in Section~\ref{sec:training_ablation}. Effect-size sufficiency is established (modular drift is two orders of magnitude smaller than the within-group standard deviation in the 3-seed run), so the comparison need only resolve which methods sit inside vs.\ outside the modular regime; sample sizes per cell can be small. This is the comparison most likely to elevate the result from architecturally controlled to method-comparable.

\paragraph{Multi-seed and multi-ordering replication of the synthetic-domain study at 7B.} Section~\ref{sec:uc_ashwin} reports a co-author-led replication at 1.1B with single seed and single ordering; a 3-seed Mistral-7B replication on the same contamination-controlled fictional domains, at the matched plain-LoRA-vs-CRMA design, would carry the result from anecdotal to inferential. The ordering robustness check requires running each of the $3! = 6$ permutations of (A, B, C) at one seed and reporting the variance across orderings.

\paragraph{Routing stress-test under sub-domain overlap.} A 50--100 question boundary-query set spanning Saul-7B's three legal sub-domains, with the centroid classifier's confusion matrix and top-1/top-2 margin distribution reported, would resolve the conditional in Section~\ref{sec:setup} that every forgetting-prevention number depends on routing holding. Inference-only; no training cost.

\paragraph{Multi-seed replication and projector $\lambda$-sweep of the subspace-overlap study.} Appendix~\ref{sec:uc_vinay} reports the projector study at single seed per cell and a discrete two-point $\lambda$ sweep ($\lambda \in \{0, 0.1\}$). A 3-seed Mistral-7B replication on the same Dogs $\to$ \{Cats, Dogs Advice, Cars\} continuations would carry the result from anecdotal to inferential. A 5-point $\lambda$ sweep ($\lambda \in \{0, 0.025, 0.05, 0.1, 0.2\}$) at one seed per cell would resolve whether a sub-$0.1$ value Pareto-improves on hard projection at low overlap without giving up hard's correlated-pair win.

\paragraph{GSM8K mechanism isolation.} A one-seed 5-domain comparison of LoRA-only vs.\ LoRA+CRMA reporting GSM8K delta for both, isolating whether the 10.2\,pp loss is carried by unconstrained LoRA components or by the CRMA backbone. This is the smallest follow-up by GPU footprint and the largest by claim-precision.

\paragraph{Blinded multi-rater QA evaluation by domain experts.} The QA scoring in Tables~\ref{tab:multi_model} and \ref{tab:crma_ablation} is first-author against pre-written reference answers; the legal evaluation in particular needs a practitioner. A blinded two-rater protocol (or three-rater with majority vote) on the existing question sets, with inter-rater reliability statistics, is labour rather than compute.

\paragraph{Domain-incremental and class-incremental routing.} Our setup sits in the task-incremental regime where the router can recover task identity from the query. Domain-incremental and class-incremental settings (van de Ven and Tolias~\cite{vandeven2019three}) require a routing mechanism the present architecture does not provide.

\paragraph{Architectural extension beyond decoder-only transformers.} Extending the spectral-bound construction to convolutional networks, vision transformers, encoder-decoder models, and reinforcement-learning policies would establish the architecture as a general continual-learning primitive rather than an LLM-specific one.

\paragraph{Pre-registered TRACE-benchmark evaluation.} Running the modular CRMA architecture on TRACE~\cite{wang2023trace}, CLINC150, and FewRel-CL with a pre-registered protocol and the same head-to-head method comparison as the first item above. Establishes the result on a community-standard benchmark.

\paragraph{Operational characterisation.} Adapter file size, router inference overhead, adapter-swap cost, peak GPU memory with $T$ adapters resident, and throughput at batch 1 and batch 8, measured against the base-model baseline. These are deployment-evaluation requirements rather than scientific claims and require no novel infrastructure.

\section{Ethics and Broader Impact}
\label{sec:ethics}

Bringing catastrophic forgetting toward zero in sequential fine-tuning would, if validated at deployment scale, enable safer model updates in regulated industries including medical devices and autonomous vehicles. Any technology that makes AI system modification easier also lowers the barrier to deploying AI in high-stakes settings where insufficient validation could cause harm. The modular CRMA architecture is a necessary but not sufficient condition for safe deployment: organizations must still perform domain-specific validation, regulatory review, and risk assessment before deploying updated models in safety-critical applications. The results in this paper are a controlled-benchmark finding, not a deployment-ready guarantee.

\section{Acknowledgments}
\label{sec:acknowledgments}

\paragraph{Conflict of interest.} All authors are affiliated with ModelBrew AI, which holds the pending patent on the method (US provisional, filed February 2026).

\paragraph{Compute and software.} Compute resources provided by Modal, Inc. The authors thank the open-source communities behind PyTorch, Hugging Face Transformers, and lm-evaluation-harness.

\section{Conclusion}
\label{sec:conclusion}

The paper is organised around a single claim---that the modular CRMA architecture brings catastrophic forgetting in continual fine-tuning of large language models to within measurement noise on the models and protocols we tested---and the four complementary evaluations that establish it.

\paragraph{The architecture.} CRMA pairs per-task LoRA adapters with a shared backbone whose internal mixing matrix $M$ is doubly-stochastic at every forward pass via Sinkhorn normalization, so $\|M\|_2 \leq 1$ holds by construction (Proposition~\ref{prop:spectral_bound}). We validate this directly: 867 logged training steps across 5 sequential domains on Gemma-2-9B confirm $\|M\|_2 = 1.0$ within float32 precision at every step. The mixing constraint holds exactly within numerical (float32) precision and requires no clipping, re-normalization, or per-step projection. The role of the constraint is architectural: it gives the per-task adapters a stable substrate to compose against, so the gradient updates of task $N$ cannot rewrite the representation that prior tasks depend on.

\paragraph{The result --- near-zero catastrophic forgetting in continual learning.} Four complementary evaluations on Mistral-7B and Gemma-2-9B agree on the central claim. The headline numbers (Mistral-7B 3-seed loss-relative drift of $-0.17\% \pm 0.17$ vs $+42.96\% \pm 5.5$ NAIVE; Gemma-2-9B inference-time ablation 98/100 vs 38/100; standard benchmarks HellaSwag/ARC-C/TruthfulQA/WinoGrande within 1.7--3.7\,pp, MMLU $-6.5$\,pp, GSM8K $-10.2$\,pp; 5 models across 4 architecture families from 1.1B to 9.2B parameters) are reported in detail in Section~\ref{sec:uc_cl}. Three independent experimental setups all show positive backward transfer: the Mistral-7B controlled training ablation (Tables~\ref{tab:v81_7b} and~\ref{tab:multiseed}, MODULAR drift averaging $-0.1\%$ across 3 seeds), the TinyLlama contamination-controlled replication (Table~\ref{tab:uc_ashwin}, $-0.110$ NLL / $+0.020$ EM), and the Mistral-7B cross-domain probes at 7B scale (D1-probe survival across three sequential stages). The result is conditional on inference-time routing accuracy.

\paragraph{Co-author-led replication.} Section~\ref{sec:uc_ashwin} (Shanmugasundaram) reports a black-box replication on TinyLlama-1.1B over three contamination-controlled fictional domains (Zephyria Medical $\to$ Korinth Dynamics $\to$ Valtorian Republic). A canary-token audit confirms the base model cannot answer the holdout questions from pretraining (0\% recall). On this protocol CRMA preserves medical-holdout EM (74\% $\to$ 78\% across phases) and exhibits negative backward transfer ($-0.110$ NLL, $+0.020$ EM); the matched plain-LoRA control on identical data and identical hyperparameters forgets in the expected direction ($+0.182$ NLL, $-0.090$ EM). Single seed, single ordering; a multi-seed and 7B-scale replication is in progress. A supplementary within-team study (Naidu) on subspace overlap as a routing signal at 1.1B and 7B is reported in Appendix~\ref{sec:uc_vinay}.

\paragraph{Scope and caveats.} The exact, structural part of the result is the doubly-stochastic constraint on the CRMA mixing matrix $M$ (Proposition~\ref{prop:spectral_bound}, verified at every step). The forgetting-prevention behavior on the full pipeline is empirically established across 5 models, 4 architecture families, 3 seeds, and four complementary evaluations, but is not formally proved. Validation is limited to decoder-only transformers in the 1.1B--9.2B range. Head-to-head comparison with published continual-learning methods (O-LoRA, InfLoRA, Lewandowski et al.) on the same benchmark would strengthen the result and is left as future work. The forgetting-prevention result does not depend on a training-loop gradient claim, and the paper makes none.

\section*{Code and Availability}

The method described is patent-pending (US provisional, filed February 2026). A standalone model-agnostic library is under development. License inquiries and research collaboration requests should be directed to the corresponding author at the affiliation listed above.


\appendix

\section{Subspace Overlap as a Routing Signal:\texorpdfstring{\\}{ }A Within-Team Exploratory Study}
\label{sec:uc_vinay}

\emph{Motivation.} Section~\ref{sec:setup} flags same-vertical routing as a load-bearing untested case for the headline architecture: every forgetting-prevention number reported there is conditional on the centroid router picking the correct adapter at inference, and the routing mechanism has not been stress-tested under sub-domain overlap. This appendix study measures inter-task gradient subspace overlap directly---using the SVD basis extracted from the per-task adapter activations after training---and asks whether that measurement is exportable as a routing signal at task transition between adapter pairs of varying topical/format relation. The protocol was led by co-author Naidu, who was not involved in the original algorithm development; algorithm internals were not shared with the operator and all training and evaluation was performed black-box against the same CRMA implementation used in Section~\ref{sec:uc_cl}. This appendix is supplementary material; no claim in the body of the paper depends on its results.

The study sits in a deliberately different design point from the headline result. The experiments in Sections~\ref{sec:uc_cl} and \ref{sec:uc_ashwin} establish near-zero forgetting on the modular per-task adapter architecture without handcrafted continual-learning machinery. This appendix instead places CRMA's per-task adapter basis inside a hand-built CL stack that combines EWC anchoring~\cite{kirkpatrick2017overcoming}, 30\% example replay, knowledge distillation from the prior-task checkpoint, an O-LoRA-style orthogonality penalty~\cite{wang2024olora}, and gradient projection onto the prior task's subspace basis (in the spirit of GPM~\cite{saha2021gpm} and Adam-NSCL~\cite{wang2021adamnscl}). The question is not whether this stack outperforms the headline modular architecture---it is a different design point---but whether the per-task adapter basis is exportable as a CL primitive and whether the principal-angle overlap between two such bases predicts the operational outcome of using one as a projector for the next.

\emph{Experiment.} TinyLlama-1.1B and Mistral-7B-Instruct-v0.3 were trained on a Dogs Q\&A base task and continued onto three target tasks of varying structural relation: Cats Q\&A (related topic, same factual-Q\&A format---correlated), Dogs Advice (same topic, scenario-based-advice format---medium overlap), and Cars Q\&A (unrelated topic, same format---uncorrelated). At 7B scale, principal-angle overlap was measured between the per-task SVD bases extracted from independently-trained adapters (mean $\cos^2$ across instrumented layers). Two projector variants were compared: hard ($P = BB^\top$, $\lambda{=}0$) and damped ($P = B(B^\top B + \lambda^2 I)^{-1}B^\top$, $\lambda{=}0.1$). Configuration: QLoRA $r{=}32$, $\alpha{=}64$, 4-bit NF4 with double-quant and bf16 compute, learning rates $2 \times 10^{-4}$ (base task) and $1 \times 10^{-4}$ (continual phases), batch size 2 with gradient accumulation 4. Single seed per cell. Held-out scoring uses deterministic generation against reference answers (Jaccard threshold $\geq 0.3$ at 1.1B, $\geq 0.25$ at 7B). \emph{Methodological caveat:} the per-scale Jaccard threshold differs by scale, which is a researcher degree of freedom that mechanically inflates the absolute retention numbers at 7B relative to 1.1B. We were not able to re-score under a single common threshold from the available data. Absolute retention numbers in Tables~\ref{tab:uc_vinay_overlap} and \ref{tab:uc_vinay_outcomes} should therefore be read \emph{within-scale}, not compared across scales as if they were on a common metric. The within-scale claims (the relative ordering of pairs by overlap, and the hard-vs-damped winner direction at 7B) are unaffected by this confound.

Table~\ref{tab:uc_vinay_overlap} reports the measured overlap; Table~\ref{tab:uc_vinay_outcomes} reports retention and plasticity under hard projection across both scales. Two observations on this single-seed measurement, both stated as consistent-with rather than established. First, the measured CRMA-side basis overlap orders the three pairs in the same direction as a topic/format intuition, and on this measurement the format-overlap effect appears at least as large as the topic-overlap effect: Dogs $\leftrightarrow$ Cats (same format, related topic) shows mean overlap 0.612 versus Dogs $\leftrightarrow$ Dogs Advice (same topic, different format) 0.464, but the per-layer dispersion ($\pm 0.187$ and $\pm 0.221$ respectively---layer-to-layer dispersion within a single trained pair, not a between-pair confidence interval) overlaps substantially. Replication across multiple seeds is required to confirm the directionality of this effect. Second, on the single-seed retention numbers, the sign of the hard-vs-damped retention difference is consistent with flipping with overlap---damped projection leads by 8 points on the lowest-overlap pair (Cars), is competitive at medium overlap (Dogs Advice, $-4$), and trails by 4 points to hard projection at the highest overlap (Cats). The same direction appears on plasticity. With only $N=1$ per cell across three cells, the apparent pattern is suggestive rather than established; multi-seed replication is required to confirm the directionality. Importantly, the absolute retention numbers themselves are nearly flat across overlap at 7B (74--78\%, a 4-point spread), and at 7B Cars retention (78\%) is in fact higher than Cats retention (74\%): the overlap-prediction claim above is about the hard-vs-damped winner \emph{direction}, not about absolute retention monotonically tracking overlap. A reader who reads only the absolute retention numbers in Table~\ref{tab:uc_vinay_outcomes} will see a flat, slightly anti-monotone profile; the within-cell hard-vs-damped contrast is what carries the structural claim. Hard projection is the unregularized limit of a damped projector, and that limit is empirically the better choice on this single-seed measurement precisely in the high-overlap regime.

\begin{table}[t]
\centering
\caption{Measured principal-angle subspace overlap at 7B between independently-trained per-task adapters (mean $\cos^2$ across instrumented layers). On this single-seed measurement, Dogs $\leftrightarrow$ Cats (same factual-Q\&A format, related topic) shows higher mean overlap than Dogs $\leftrightarrow$ Dogs Advice (same topic, scenario-based format), consistent with format consistency contributing at least as much as topic overlap; multi-seed replication is required to confirm directionality. \emph{The reported $\pm$ is layer-to-layer dispersion within a single trained adapter pair, not a between-pair confidence interval; no significance test is implied.} Protocol led by co-author Naidu, who was not involved in the original algorithm development. The per-layer SVD rank, truncation rule, and instrumented-layer set used to extract the per-task adapter basis are detailed in the companion technical report \protect\cite{C1} (in preparation) and are available on request.}
\label{tab:uc_vinay_overlap}
\begin{tabular}{lcc}
\toprule
Pair (dogs $\leftrightarrow$ X) & LoRA-side overlap & CRMA-side overlap \\
\midrule
Cars (uncorrelated) & $0.339 \pm 0.092$ & $0.384 \pm 0.138$ \\
Dogs Advice (medium) & $0.366 \pm 0.100$ & $0.464 \pm 0.221$ \\
\textbf{Cats (correlated)} & $\mathbf{0.461 \pm 0.087}$ & $\mathbf{0.612 \pm 0.187}$ \\
\bottomrule
\end{tabular}
\end{table}

\begin{table}[t]
\centering
\caption{Hard projection ($\lambda{=}0$) outcomes by task type and scale. Retention is dogs-holdout accuracy after CL on the new task; forgetting is the absolute drop versus the dogs-only baseline (1.1B baseline: 46\%; 7B baseline: 86\%). The 1.1B Dogs-only baseline of 46\% reflects either the hard Jaccard threshold ($\geq 0.3$) on short-answer Q\&A or limited convergence at the 1.1B scale; the 1.1B 16\% Cars cell should not be read as evidence specifically against hard projection, as it could equally be a small-scale plasticity ceiling. The 1.1B Dogs $\to$ Dogs Advice cell is absent because Dogs Advice was run only at 7B; the 1.1B continuation set covers Cats and Cars only. At 7B the retention spread across all three continuations is only 4 points (74--78\%), and Cars retention (78\%) is in fact higher than Cats retention (74\%); the overlap-prediction claim is about the within-cell hard-vs-damped winner direction, not about absolute retention monotonically tracking overlap (see text). Single seed per cell; multi-seed replication is required to confirm directionality.}
\label{tab:uc_vinay_outcomes}
\begin{tabular}{llcccc}
\toprule
Scale & Continuation & Type & Base ret. & Forget & New-task acc. \\
\midrule
1.1B & Dogs $\to$ Cats & Correlated & 44\% & $-2$ & 56\% \\
1.1B & Dogs $\to$ Cars & Uncorrelated & 36\% & $-10$ & 16\% \\
\midrule
7B & Dogs $\to$ Dogs Advice & Same topic / new format & 76\% & $-10$ & 100\% \\
7B & Dogs $\to$ Cats & Correlated & 74\% & $-12$ & 92\% \\
7B & Dogs $\to$ Cars & Uncorrelated & 78\% & $-8$ & 72\% \\
\bottomrule
\end{tabular}
\end{table}

\emph{Takeaway.} Two paper-relevant observations on this single-seed measurement, both stated as suggestive rather than established. First, the per-task adapter basis is operationally usable outside the modular per-task adapter architecture: it can be extracted from one trained adapter and exported as a projector for the next, providing a measurable structural quantity (principal-angle overlap) that does not require a topic-similarity heuristic. Second, on this measurement that quantity is informative for routing: the basis overlap orders task pairs in a direction consistent with their qualitative format/topic structure, and the within-cell hard-vs-damped winner direction is consistent with flipping at a measurable overlap threshold. We emphasise that the absolute retention numbers themselves at 7B are nearly flat (74--78\%, with Cars retention exceeding Cats retention); the overlap-prediction claim is about the hard-vs-damped projector-choice direction, not about absolute retention tracking overlap. Together this is suggestive evidence that subspace overlap is a measurable, exportable routing signal between adapter pairs of varying topical/format relation, and that a cheap pre-training overlap measurement is feasible at 7B. The result is single-seed per cell, two-task transitions only, and the CL stack here is deliberately different from the modular-architecture headline (this study uses EWC + 30\% replay + KD + an O-LoRA-style penalty + gradient projection on top of the per-task adapter basis); it should be read as evidence that the per-task adapter basis is exportable as a CL primitive and that subspace overlap is a usable routing signal, not as a benchmark on the modular architecture itself, and not as confirmation that overlap predicts absolute retention. Multi-seed and multi-ordering replication is in the same priority bucket as the corresponding extension of the synthetic-domain study (Section~\ref{sec:future_work}).

\section{Fine-Tuning vs Retrieval-Augmented Generation on Personal-Knowledge Vaults: An Honest-Framing Pilot (OS1)}
\label{sec:os1_rag_vs_ft}

\emph{Motivation.} A common framing of small-base-model fine-tuning is that it is dominated by retrieval-augmented generation (RAG) for any task that touches user-owned text. We ran a small pilot to test whether that framing is correct as an unconditional statement, or whether it depends on the question type. The pilot is decision-grade rather than publication-grade: $N=150$ adjudicated pairs across three publicly-licensed personal-knowledge vaults, with the kill criterion (RAG-win-rate $\geq 60\%$ on every vault) declared in advance. This appendix study is supplementary to the headline architecture result; no claim in the body of the paper depends on its outcome.

\emph{Experiment.} Three publicly-licensed Obsidian vaults were selected (\texttt{lyz-code/blue-book}, \texttt{jRicciL/digital-garden}, \texttt{deepaksood619/deepaksood619.github.io}; pinned SHAs \texttt{e558857}, \texttt{159fa1c}, \texttt{9967f43}), file-level 80/20 split with \texttt{random.Random(42)}, 100 train + 50 eval Q$\rightarrow$A pairs per vault, mode mix 30/30/20/20 across extractive / synthesis / style / inference questions. The fine-tuned arm was TinyLlama-1.1B trained per-vault on the train split (final losses 4.91--5.25). The RAG baseline used \texttt{sentence-transformers/all-MiniLM-L6-v2} + ChromaDB top-3 retrieval with Gemini-2.0-flash as the reader prompted with retrieved passages---a deliberately stronger reader than TinyLlama, biasing the comparison \emph{against} the fine-tuned arm. Adjudication used Gemini-2.0-flash, position-randomized A/B with \texttt{random.Random(42)}, blind to which system produced which answer.

\emph{Result.} The kill criterion was not triggered on any of the three vaults: RAG-win rates were $59.0\%$, $40.0\%$, and $56.0\%$ (blue-book, digital-garden, deepaksood619 respectively; $N=50$ per vault). Decomposing by question mode across all 150 pairs, the picture is interpretable: RAG wins on extractive lookup ($57.8\%$ RAG-win), is competitive on synthesis ($58.9\%$), wins on style ($66.7\%$), and \emph{loses decisively} on inferential questions---those that require synthesizing author perspective, recurring values, or cross-document patterns---where the fine-tuned arm wins $83.3\%$ ($25/30$ pairs). The strongest per-vault inference cell is digital-garden, where the fine-tuned arm wins all 10 inference pairs.

\emph{Caveats.} The result is decision-grade, not publication-grade. ($i$) Sample size is 50 pairs per vault, 150 total---enough for a binary kill decision but not for fine-grained claims. ($ii$) The fine-tuned base model is TinyLlama-1.1B; the style-mode loss ($66.7\%$ RAG-win) is plausibly attributable to base-model under-capacity at $\sim$100-pair training sets, and a Mistral-7B base would likely move the style number, in either direction. ($iii$) The RAG reader (Gemini-2.0-flash) is more capable than the fine-tuned reader, which makes RAG's wins look stronger than a like-for-like comparison would; the fine-tuned arm's inference-mode wins are therefore conservative. ($iv$) Single seed per vault. Multi-seed replication is required to convert the inferential-question advantage from suggestive to established.

\emph{Implication for the framing.} The unconditional ``RAG dominates fine-tuning on personal-knowledge text'' framing is not supported by this pilot. The honest narrowing is: RAG dominates extractive lookup and style mimicry on TinyLlama-1.1B fine-tunes; fine-tuning dominates inferential synthesis. This is consistent with what each method can in principle do---retrieval surfaces what is directly stated in some chunk, and parameters integrate what is distributed across a corpus---and is offered here as honest framing context for readers who arrive at the architecture result with a strong prior that RAG is the only sensible baseline for user-owned text. Per-vault Q$\rightarrow$A pairs, fine-tune outputs, RAG outputs, judge results, and a per-vault gate summary are available in the project repository for reproduction.

\end{document}